\def\year{2020}\relax
\documentclass[letterpaper]{article} 
\usepackage{aaai20}  
\usepackage{times}  
\usepackage{helvet} 
\usepackage{courier}  
\usepackage[hyphens]{url}  
\usepackage{graphicx} 
\usepackage{graphicx}  
\title{
Infinity Learning: Learning Markov Chains from\\ Aggregate Steady-State Observations
}
\author{
\Large \textbf{Jianfei Gao\textsuperscript{\rm 1}, Mohamed A. Zahran\textsuperscript{\rm 2}, Amit Sheoran\textsuperscript{\rm 3}, Sonia Fahmy\textsuperscript{\rm 4}, Bruno Ribeiro\textsuperscript{\rm 5}} \\
Department of Computer Science, Purdue University \\
305 N. University St, West Lafayette, IN 47907 \\
\{gao462\textsuperscript{\rm 1}, mzahran\textsuperscript{\rm 2}, asheoran\textsuperscript{\rm 3}\}@purdue.edu, \{fahmy\textsuperscript{\rm 4}, ribeiro\textsuperscript{\rm 5}\}@cs.purdue.edu \\
}

\usepackage{xcolor}
\usepackage{xparse}
\usepackage{amsthm}
\usepackage{xspace}
\usepackage{graphicx}
\usepackage{amsmath,amsfonts,bm}
\usepackage{subcaption}
\usepackage{dsfont}
\usepackage{booktabs}
\usepackage{algorithm}
\usepackage{algorithmic}
\usepackage{pgf}
\usepackage{tikz}
\usetikzlibrary{arrows,automata}

\newcommand{\Appendix}{Supplementary Material\xspace}

\newcommand{\AppEgA}{\Appendix A1\xspace}
\newcommand{\AppEgB}{\Appendix A2\xspace}
\newcommand{\AppEgC}{\Appendix A3\xspace}
\newcommand{\AppProA}{\Appendix B1\xspace}
\newcommand{\AppProB}{\Appendix B2\xspace}
\newcommand{\AppProC}{\Appendix B3\xspace}
\newcommand{\AppExpA}{\Appendix C1\xspace}
\newcommand{\AppExpB}{\Appendix C2\xspace}
\newcommand{\AppExpC}{\Appendix C3\xspace}
\newcommand{\AppExpD}{\Appendix C4\xspace}
\newcommand{\AppExpE}{\Appendix C5\xspace}

\newcommand{\citet}[1]{\citeauthor{#1} \shortcite{#1}}
\newcommand{\citep}{\cite}

\newcommand{\Puni}{{\bm{P}}}

\newtheorem{definition}{Definition}
\newtheorem{proposition}{Proposition}
\newtheorem*{proposition*}{Proposition}
\newtheorem{theorem}{Theorem}
\newtheorem*{theorem*}{Theorem}

\newtheorem*{corollary*}{Corollary}
\newtheorem{remark}{Remark}
\newtheorem{lemma}{Lemma}
\newtheorem*{lemma*}{Lemma}

\newcommand{\cD}{\mathcal{D}}

\newcommand{\cY}{\mathcal{Y}}

\makeatletter
\let\overlinewithoriginalheight\overline
\newcommand*\overlinewithlessheight[1]{{\mathpalette\overline@aux{#1}}}
\newcommand*\overline@aux[2]{
  \begingroup
    \count0=\fam 
    \setbox0=\hbox{$\m@th #1\fam=\count0 #2$}
    \@tempdima=.4\ht0
    \setbox0=\hbox{$\m@th #1\fam=\count0\overlinewithoriginalheight{#2}$}%
    \advance\@tempdima by .6\ht0
    \ht0=\@tempdima 
    \usebox0
  \endgroup%
}
\let\overline\overlinewithlessheight

\let\underlinewithoriginaldepth\underline
\newcommand*\underlinewithlessdepth[1]{{\mathpalette\underline@aux{#1}}}
\newcommand*\underline@aux[2]{%
  \begingroup
    \count0=\fam
    \setbox0=\hbox{$\m@th #1\fam=\count0 #2$}%
    \@tempdima=.4\dp0%
    \setbox0=\hbox{$\m@th #1\fam=\count0\underlinewithoriginaldepth{#2}$}%
    \advance\@tempdima by .6\dp0%
    \dp0=\@tempdima
    \usebox0%
  \endgroup%
}
\let\underline\underlinewithlessdepth
\makeatother

\newcommand{\harrow}[1]{\mathstrut\mkern2.5mu#1\mkern-11mu\raise1.6ex\hbox{$\scriptscriptstyle\rightharpoonup$}}

\def\Figref#1{Figure~\ref{#1}}

\def\Eqref#1{Equation~\eqref{#1}}

\def\1{\bm{1}}

\def\rvp{{\mathbf{p}}}

\def\rmY{{\mathbf{Y}}}

\def\vtheta{{\bm{\theta}}}

\def\vx{{\bm{x}}}
\def\vy{{\bm{y}}}

\def\mI{{\bm{I}}}

\def\mP{{\bm{P}}}
\def\mQ{{\bm{Q}}}

\DeclareMathAlphabet{\mathsfit}{\encodingdefault}{\sfdefault}{m}{sl}
\SetMathAlphabet{\mathsfit}{bold}{\encodingdefault}{\sfdefault}{bx}{n}

\def\sN{{\mathbb{N}}}

\def\sS{{\mathbb{S}}}

\def\sZ{{\mathbb{Z}}}

\newcommand{\expected}{\mathbb{E}}

\newcommand{\bpi}{\boldsymbol{\pi}}

\newcommand{\Loss}{\mathcal{L}}

\NewDocumentCommand{\Prob}{o d()}{\ensuremath{
    \mathds{P}[#1]
}}

\DeclareMathOperator*{\argmin}{arg\,min}

\newcommand{\eg}{\emph{e.g.}}
\newcommand{\ie}{\emph{i.e.}}

\newcommand{\mix}[1]{\text{mix}\left(#1\right)}

\newcommand{\proba}{\text{Pr}}

\newcommand\melepad[1]{\makebox[3.5em]{$#1$}}

\begin{document}

\maketitle

\begin{abstract}
We consider the task of learning a parametric Continuous Time Markov Chain (CTMC) sequence model without examples of sequences, where the training data consists entirely of aggregate steady-state statistics. Making the problem harder, we assume that the states we wish to predict are unobserved in the training data.
Specifically, given a parametric model over the transition rates of a CTMC and some known transition rates, we wish to extrapolate its steady state distribution to states that are unobserved. 
A technical roadblock to learn a CTMC from its steady state has been that the chain rule to compute gradients will not work over the arbitrarily long sequences necessary to reach steady state ---from where the aggregate statistics are sampled. To overcome this optimization challenge, we propose $\infty$-SGD, a principled stochastic gradient descent method that uses randomly-stopped estimators to avoid infinite sums required by the steady state computation, while learning even when only a subset of the CTMC states can be observed. We apply $\infty$-SGD to a real-world testbed and synthetic experiments showcasing its accuracy, ability to extrapolate the steady state distribution to unobserved states under unobserved conditions (heavy loads, when training under light loads), and succeeding in difficult scenarios where even a tailor-made extension of existing methods fails.
\end{abstract}

%
%
%
%
\section{Introduction} \label{sec:Introduction} 
{\em Can we learn a parametric sequence model given only aggregate statistics as training data?}
As machine learning expands into new applications, new learning paradigms emerge, such as learning a sequence model from a set of observations without any clear time order between them.

Traditional supervised and unsupervised learning methods are essentially tasked with problems that can be learned from examples ({\em interpolation}).
In a host of key applications of parametric sequence models, we want to {\em extrapolate}, \ie, take these aggregate observations and extrapolate them to a scenario {\em not observed} in the training data.

For instance, servers in the cloud collect system logs ---aggregate statistics such as response-time distribution, queue length distribution--- 
under light-load conditions.
Under high-loads, however, these servers may disable statistics collection (logs) due to the potential performance penalty of logging~\citep{logperf}.
Capacity planning requires knowing how the servers perform under medium to high load conditions, which requires extrapolated predictions of request loss probability and server response times from the collected light-load data.

Hence, in this work we consider the task of learning a parametric Continuous Time Markov Chain (CTMC) sequence model ---with transition rate matrix $\mQ(\vx,\vtheta)$ where $\vx$ are known parameters but parameters $\vtheta$ must be learned--- without examples of sequences, where the training data consists entirely of aggregate steady-state statistics.
Making the problem harder, we further assume that the states we wish to predict are unobserved in the training data.
More specifically, given an inductive bias over the transition rates of a CTMC and some known transition rates, we wish to extrapolate its steady state distribution to states that are unobserved. 
We focus on the application of predicting failures in queuing systems under heavy loads ---\eg, predicting request loss rates in overloaded cloud services--- with training data that contains only aggregate statistics of the system under light loads, and no observed losses. 
Traditionally, CTMCs are learned from observations of their transient (sequences given by transitions between states) not from observations of their steady state, even less so if only a subset of the state space is observable.

\begin{remark}
Extrapolation v.s.\ generalization error: In our task we must make a distinction between {\em generalization error} ---which is the error on unseen data that reduces with more training examples even without inductive biases--- and {\em extrapolation error}~\citep{marcus1998rethinking} ---which is a type of generalization error over unseen states and domains that does not reduce with more training data without the help of a modeling assumption.
Our task is to learn a parametric model that is capable of extrapolation.
\end{remark}

\vspace{-8pt}
\paragraph{Contributions.}
Our work introduces the general problem of learning a {\em parametric} CTMC from aggregate steady-state observations (frequencies) of part of the CTMC states, focusing on queueing systems as our application.
We also introduce a novel method ($\infty$-SGD) to learn parametric CTMCs from aggregate steady-state observations, which work even if the observations are over a restricted set of states.
Our approach, $\infty$-SGD, is a novel, theoretically principled, optimization approach that, among other things, uses randomly-stopped estimators~\citep{mcleish2011general}.
In our experiments $\infty$-SGD finds significantly better maximum likelihood estimates than the baselines in real testbed and synthetic scenarios, both for the training and test data.
We also see that $\infty$-SGD can successfully extrapolate from training data under light queueing loads to predictions under heavy loads.
We expect $\infty$-SGD to be a useful tool in applications that collect aggregate statistics but need to learn parametric CTMCs.

%
%
%
%
\section{Preliminaries} \label{sec:Preliminaries} 

Consider a stationary and ergodic Continuous-Time Markov Chain (CTMC) $\cY = (\rmY_\tau)_{\tau \geq 0}$ over a finite state space $\sS$, where $\rmY_{\tau}$ is the state of the Markov chain at time $\tau$.
The CTMC is governed by Kolmogorov's Forward Equation
\begin{equation} \label{eq:Kolm}
    \frac{\partial}{\partial \tau} \rvp_{\vx,\vtheta}(\tau)^\mathsf{T} = (\rvp_{\vx,\vtheta}(\tau))^\mathsf{T} \mQ(\vx,\vtheta) ,
\end{equation}
where 
$\mQ(\vx;\vtheta)$ is a transition rate matrix parameterized by both $\vx$ (a vector of observed parameters, \eg, request rate) and $\vtheta$ (a vector of hidden parameters), 
$\rvp_{\vx,\vtheta}(\tau)$ is a column vector of dimension $\vert \sS \vert$, with $\proba[\rmY_\tau = i] = \rvp_{\vx,\vtheta}(\tau)_{i}$ as the probability of being at state $i \in \sS$ at time $\tau \geq 0$, given that $\cY$ starts at state $j \in \sS$ with probability $\proba[\rmY_0 = j] = \rvp(0)_j$. 

The transition rate matrix $\mQ(\vx,\vtheta)$ is such that for $i \neq j$, $(\mQ(\vx,\vtheta))_{ij} \geq 0$ describes the rate of the process transitions from state $i$ to state $j$.
The diagonal $(\mQ(\vx,\vtheta))_{ii}$ is such that each row of $\mQ(\vx,\vtheta)$ sums to zero, irrespective of the values of $\vx$ and $\vtheta$.
Because $\cY$ is stationary and ergodic, the solution to \Eqref{eq:Kolm} implies a unique steady state distribution $\bpi(\vx,\vtheta) = \lim_{\tau \to \infty} \rvp_{\vx,\vtheta}(\tau)$.

\vspace{-5pt}
\paragraph{Parameterized transition rate matrix $\mQ(\vx,\vtheta)$.}
We exemplify $\mQ(\vx;\vtheta)$ with one of the simplest CTMCs: the birth-death process (BD). BD has two parameters: the request (birth) rate $\vx = (\lambda)$ and the service (death) rate $\vtheta = (\mu)$.
The transition rate matrix is
{\small \begin{equation*}
	\mQ(\vx; \vtheta) = \begin{bmatrix}
	-\lambda & \lambda & 0 & \cdots & 0 \\
	\mu & -(\mu + \lambda) & \lambda & \cdots & 0 \\
	0 & \mu & -(\mu + \lambda) & \cdots & 0 \\
	\vdots & \vdots & \vdots & \ddots & \vdots \\
	\cdots & \cdots & \cdots & \cdots & -\mu 
	\end{bmatrix},
\end{equation*}
}
\hspace{-4pt}where request rate $\lambda$ is known but the service rate $\mu$ needs to be learned.
In our work, $\mQ(\vx,\vtheta)$ can be significantly more complex, as we only assume $\mQ(\vx,\vtheta)$ is differentiable w.r.t.\ $\vtheta$.
More generally, we can have an $n \times n$ matrix
\begin{equation*}
\resizebox{1.\hsize}{!}{$
	\mQ(\vx; \vtheta) = \begin{bmatrix}
	-\sum_{i \neq 1} f_{1,i}(\vx,\vtheta) & f_{1,2}(\vx,\vtheta) & \cdots & f_{1,n}(\vx,\vtheta) \\
	f_{2,1}(\vx,\vtheta) & -\sum_{i \neq 2} f_{2,i}(\vx,\vtheta) & \cdots & f_{2,n}(\vx,\vtheta) \\
	\vdots & \vdots & \ddots  & \vdots \\
	f_{n,1}(\vx,\vtheta) & f_{n,2}(\vx,\vtheta) & \cdots  & -\sum_{i \neq n} f_{n,i}(\vx,\vtheta) \\
	\end{bmatrix},
$}
\end{equation*}
for some appropriate set of functions $\{f_{i,j}\}_{i,j}$ of $\vx$ and $\vtheta$ (whose image must be in $[0,\infty)$).

\vspace{-5pt}
\paragraph{Learning task.}
Consider learning $\vtheta$ from a set of steady-state observations from a subset $\sS' \subseteq \sS$ of the states of the CTMC $\cY$.
That is, even though $\cY$ evolves over $\sS$, the observations from states in $\overline{\sS}' = \sS \backslash \sS'$ are unavailable to us ---e.g., consider a system that disables statistics collection (logs) when it reaches a set of {\em system overload} states $\overline{\sS}'$.

{\em Training data:} Our training data consists of $M$ time windows from which we have observed aggregate steady state data: $\cD = \{(\vx_m,\vy_m)\}_{m=1}^M$, where $y_{m,j} \equiv (\vy_m)_j$ is the number of steady state observations of state $j \in \sS'$ at the $m$-th time window.

{\em Loss function:}
The minimum negative log-likelihood of the model must be conditioned on only observing states of $\sS'$ in steady state (i.e., $\tau \to \infty$),
\begin{equation} \label{eq:MLE}
\begin{aligned}
    \vtheta^\star =  \argmin_{\vtheta} \sum_{m=1}^M \Loss(\vy_m, \lim_{\tau \to \infty} \rvp_{\vx,\vtheta}(\tau)),
\end{aligned}
\end{equation}
where
\begin{equation} \label{eq:loss}
\Loss(\vy,\bpi)\! = \!
     - \sum_{j \in \sS'} y_{j} \log \! \left(  \frac{\bpi_{j}}{\sum_{j' \in \sS} \bpi_{j'}} \! \right),
\end{equation}
such that the denominator ensures the observations are conditioned on only observing states in $\sS'$ ---a detailed description of the math behind this conditional can be found in~\citet{Meyer1989a}.

In theory, we could optimize $\vtheta$ in \Eqref{eq:MLE} via gradient descent but the derivative of \Eqref{eq:MLE} 
w.r.t.\ $\vtheta$ requires computing the derivative of the steady state $\lim_{\tau \to \infty} \rvp_{\vx,\vtheta}(\tau)$, which is challenging as our steady state distribution does not have a closed-form expression.

\vspace{-8pt}
\paragraph{The identifyability of $\mQ$ is irrelevant to our task:} 
In our task, we wish to predict the steady-state distribution of unobserved states from samples from the steady state of observed states.
Specifically, we wish to extrapolate those predictions such that we can predict these steady state distributions even when the observed parameters, $\vx$ of $\mQ(\vx,\vtheta)$ change.
Because there are infinitely many $\mQ$ that can give the same correct steady state distribution predictions (see \AppEgA), it is irrelevant to us knowing whether we recovered the ``true'' $\mQ$.
In fact, in our formulation there is no notion that we can ever learn a ``true'' $\mQ$. 
We only care if it gives the correct steady state distribution.

Next, we review the related work.

%
%
%
%
\section{Related Work} \label{sec:Related} 

{\em Inverting an MC steady state.}
\citet{bernstein2016consistently} is one of the most closely related works, showing an estimator for an existing optimization approach from econometrics, Conditional Least Squares (CLS)~\citep{miller1952finite,van1983estimation,kalbfleisch1984least}, which can be used to learn a Markov chain from aggregate statistics.
This approach, however, is not designed to learn a parametric model (our $\mQ(\vx,\vtheta)$ needs derivatives w.r.t.\ $\vtheta$) and thus, cannot extrapolate to unobserved states in the training data.
Moreover, our Markov chain is not homogeneous across observation time windows, requiring $\vx$ to also change, which conflicts with the assumptions in CLS.

\citet{Maystre2015PlackettLuce} and~\citet{Ragain2016Luces} are the also closely related works, which learn the transition rates of a Plackett–Luce-type model CTMC from samples of its stationary distribution.
In an earlier work, \citet{Kumar2015} learns a discrete-time Markov chain model similar to the Plackett–Luce's model in the context of Web navigation.
These earlier works, however, make domain-specific assumptions on $\mQ$ that make computing $\bpi$ from $\mQ$ trivial.
We consider a general parametric $\mQ(\vx,\vtheta)$ that may have no trivial solution.

\citet{hashimoto2016learning} uses far-apart observations of a MC to learn transition probabilities, and~\citet{pierson2018inferring} uses cross-sectional data to learn a temporal model; these works are focused on specific diffusion processes.
Our problem is also related to the more general problem of learning over distributions~\citet{szabo2016learning}, which in our scenario requires a solution designed for the task.

Randomly stopped estimators have been used in unrelated machine learning tasks~\citep{xu2019variational} and~\citep{filippone2015enabling}, with significantly different tasks and estimators than ours.
Applying randomly stopped estimators is mostly about proving that a specific estimator gives finite-variance estimates.

{\em Queueing systems.}
Cloud computing has transformed IT operations and management by deploying services on commodity hardware in public or private data centers, saving millions of dollars in both capital and operational expenses~\citep{CACM,etsiarch}. 
The savings in operational expenses can only be attained if the allocation of compute, memory, networking and storage resources scales based on the workload~\citep{ENVI,aws-scaling,google-scaler}.
A key problem in this elastic scaling is anticipating overload and failures in order to proactively allocate and initialize additional resources.
This prediction needs to be done without sufficient data on overload and failures~\citep{Weiss1998}.
Fortunately, several novel cloud computing services can be modeled by queueing systems~\citep{MMmm+r,MGmm+r}. 
Existing approaches, however, require knowing the transition rate matrix rather than learning it from aggregate observations (as we do).

%
%
%
%
\section{Learning Transition Rates from Aggregate Steady State Metrics} \label{sec:Theory}

In this section, we will describe why a good parametric model of $\mQ$ is key to learn a $\vtheta^\star$ that can predict the steady state distribution of the states $\overline{\sS}' = \sS \backslash \sS'$ that are not observed from the states that are observed $ \sS'$.
We will then introduce a few na\"ive methods to learn $\vtheta^\star$ from Equation~\eqref{eq:MLE} and show they are unsuitable for learning accurate CTMC transition rates, including an extension of BPTT.
Finally, we will propose a novel approach to learn $\vtheta^\star$ that is significantly more accurate and more computationally efficient than the na\"ive approaches.

For the ease of notation,  sometimes we abbreviate transition rate matrix that $\mQ \equiv \mQ(\vx, \vtheta)$ and we may denote $(\mQ(\vx, \vtheta))_{ij}$ by $q_{ij}$.

\vspace{-3pt}
\paragraph{The need for a good parametric model of $\mQ$:}
Without tied parameters in $\mQ$ through $\vtheta$, the steady state distribution would be flexible enough to make $\bpi_i$, $\forall i \in \sS'$, and $\bpi_j$, $\forall j \in \overline{\sS}$, have arbitrarily different probabilities. 
This would make it impossible to correctly extrapolate the observed data and predict $\bpi_j$ for observations of states in $\overline{\sS}'$.
An example is provided in \AppEgA.

\vspace{-5pt}
\paragraph{Lagrangian multipliers (\eg, Conditional Least Squares) are undesirable.}
To solve Equation~\eqref{eq:MLE}, $t \gg 1$, we can add the condition $\bpi^T \mQ(\vx,\vtheta) = 0$ as a Lagrangian multiplier as if $\bpi$ are extra learnable parameters.
Then, the loss function is redefined as
\begin{equation} \label{eq:Lagrange}
	- \sum_{j \in \sS'} y_j \log \left( \lim_{t \to \infty} \frac{\bpi_{j}}{\sum_{j' \in \sS'} \bpi_{j'}} \right) + \lambda \Vert \bpi^\mathsf{T} \mQ(\vx, \vtheta) \Vert,
\end{equation}
$\lambda > 0$,
which is the Conditional Least Squares~\citep{miller1952finite} for a CTMC, a regularization applied the definition of a steady-state distribution.
We found, however, that this approach is very challenging by design, since $\bpi$ is a function of $\vx$ and $\vtheta$, and $\vx$ varies in the training data.
Hence, the Lagrangian multiplier $\lambda$ depends on the loss function (which is conditional) and on $\vx$ and $\vtheta$, a challenging task.

Moreover, if we assume a constant $\lambda$,
the resulting approach needs to work as as a bi-level optimization procedure~\citep{bhatnagar1998two,colson2007overview}.
In computing the derivatives of the loss w.r.t.\ $\vtheta$, there is essentially no connection between the data ($\{y_j\}_{j \in \sS'}$) and $\vtheta$, which is the reason why the approach fails.
Fixing these Lagrangian multiplier issues is future work.


\subsection{Solution through Uniformization and Chain Rule}

Our first step to a solution is to {\em uniformize} the Markov chain $\mQ$ in order to transform the CTMC of \Eqref{eq:Kolm} into a discrete-time Markov chain (DTMC) with probability matrix $\mP(\mQ(\vx,\vtheta))$ as described below~\citep{uniformization}.
Hence, we will see that an approximation of the steady-state distribution can obtained by recursively applying $\mP(\mQ(\vx,\vtheta))$,
and the derivative of this recursive application of the transition probability can be obtained via chain rule.
\begin{definition}[Uniformized Markov Chain] \label{def:unif}
Let $\mQ(\vx,\vtheta)$ be a stationary and ergodic CTMC. 
We define a set of Chapman-Kolmogorov equations representing the CTMC at the events (arrivals) of a Poisson process with rate $\gamma(\vx,\vtheta) > \max(-\text{diag}(\mQ(\vx,\vtheta)))$. The distribution after $t \geq 0$ of these events is 
\begin{equation}\label{eq:UnifP}
\rvp^\text{(events)}(t;\vx,\vtheta)^\mathsf{T} =  \rvp^\text{(events)}(0)^\mathsf{T} \Puni(\mQ(\vx,\vtheta))^t,
\end{equation}
where $\rvp^\text{(events)}(0)$ is some initial distribution and $\Puni(\mQ(\vx,\vtheta)) = \mI +  \mQ(\vx,\vtheta)/\gamma(\vx,\vtheta)$, where $\mI$ is the identity matrix. 
\end{definition}
By construction, since $\cY$ is ergodic and has a steady state $\bpi$, the Markov chain described by $\Puni(\mQ(\vx,\vtheta))$ has the same steady state as the CTMC described by $\mQ(\vx,\vtheta)$~\citep{uniformization}, \ie, for $\rvp_{\vx,\vtheta}(\tau)$ as described in Equation~\eqref{eq:Kolm},
\resizebox{\linewidth}{!}{
\begin{minipage}{\linewidth}
\begin{equation} \label{eq:pi}
\begin{aligned} 
    \bpi(\vx,\vtheta) &= \lim\limits_{\tau \to \infty} \rvp_{\vx,\vtheta}(\tau) = \lim\limits_{t \to \infty} \rvp^\text{(events)}(t;\vx,\vtheta) \\
    &= \lim\limits_{t \to \infty} (\rvp^\text{(events)}(0)^\mathsf{T}  \Puni(\mQ(\vx,\vtheta))^t)^\mathsf{T}.
\end{aligned}
\end{equation}
\end{minipage}
}
%
Note that $(\Puni(\mQ(\vx,\vtheta))^t)_{ij}$ is the probability that the CTMC starts at state $i$ and reaches state $j$ after $t \geq 0$ events of the Poisson process with rate $\gamma(\vx,\vtheta)$ given by Definition~\ref{def:unif}.

In what follows, we sometimes denote the probability matrix as $\mP \equiv \mP(\mQ(\vx, \vtheta))$ and steady state distribution as $\bpi \equiv \bpi(\vx,\vtheta) \equiv \bpi(\mP(\mQ(\vx, \vtheta)))$. 

\vspace{-5pt}
\paragraph{Chain rule to learn $\vtheta$ from a steady-state approximation.}
Learning $\vtheta^\star$ in Equation~\eqref{eq:MLE} through gradient descent can be approximated for a large enough value of $t^\star \gg 1$ through the chain rule~\citep{werbos1990backpropagation}.
The derivative of the loss in Equation~\eqref{eq:MLE} is
\begin{equation} \label{eq:BPTT}
\begin{aligned}
    &\!\!\! \frac{\partial \Loss(\vy,\rvp^\text{(events)}(t^\star;\vx,\vtheta))} {\partial \vtheta_k} = \\
    &\!\!\!	\sum_{\substack{i \neq j,\\i,j \in \sS}} \left( \frac{\partial \Loss(\vy,\rvp^\text{(events)}(0)^\mathsf{T} \mP^{t^\star})} {\partial (\mP^{t^\star})_{ij}} 
    	\frac{\partial (\Puni(\mQ(\vx, \vtheta))^{t^\star})_{ij}}{\partial \vtheta_k} \right)\!, \\
\end{aligned}
\end{equation}
%
In order to compute $\partial (\Puni(\mQ(\vx, \vtheta)))^{t^\star}/\partial \vtheta_k$, we have to recursively apply the chain rule, which leads to a backpropagation through time (BPTT)-style method.

{\em BPTT challenges.} Directly using BPTT, however, has both theoretical and practical barriers.
The theoretical challenge is finding a large-enough value of ${t^\star}$ that allows $\Puni(\mQ(\vx, \vtheta))^{t^\star}$ to approximate the steady-state distribution $\bpi(\vx, \vtheta)$ for any assignment of $\mQ(\vx, \vtheta)$ that our optimization might find.
The computational challenge is both of computational resources and of numerical precision.

The following definition gives a divide-and-conquer aid to the computational challenge of calculating BPTT over $\mP(\mQ(\vx, \vtheta))^{t^\star}$:
\begin{definition}[Divide-and-Conquer BPTT (DC-BPTT)] \label{def:DCBPTT}
Assume $t^\star = 2^T$ for some $T > 1$. 
Rather than backpropagating over $t^\star$ time steps ---which is difficult if $t^\star$ is large due to vanishing and exploding gradients---, we will use a divide-and-conquer approach to reduce the backpropagation steps to $\log_2 t^\star = T$, by noting that
\begin{equation} \label{eq:baseline}
\begin{aligned}
	\mP^{2^{T}} &= 
	( ( \mP^{2^{T - 2}} )^{2} )^{2}
	&= ( \cdots ( \mP^{2} )^{2} \cdots )^{2}.
\end{aligned}
\end{equation}
That is, rather than multiplying an intermediate $\mP^t$ by $\mP$ to obtain $\mP^{t+1}$, we multiply $\mP^t$ by itself to obtain $\mP^{2t}$.
\end{definition}
\vspace{-1pt}
Computing $\mP^{t^\star}$, $t^\star = 2^T$ with $T> 1$, from Definition~\ref{def:DCBPTT} is more computationally efficient than the na\"ive $t^\star$ multiplications $\mP \cdots \mP$ because the computation graph is a tree whose backpropagation paths from the root to the leaves give the same derivatives at the same tree height.
Unfortunately, as we see in our experiments, DC-BPTT still fails in the most challenging tasks.

\subsection{Solution via Infinity Learning}
An alternative to BPTT is to dive deeper into the chain rule equations and look for mathematical equivalences. 
Rather than using BPTT to compute the gradient $\partial \Puni(\mQ(\vx, \vtheta))^{t^\star}\!\!/\partial \vtheta_k$, $k \in \sS$, in Equation~\eqref{eq:BPTT}, we can make use of the following observation.
\begin{lemma} \label{lem:deriv}
Let $\mQ(\vx,\vtheta)$ be a $K$-state transition matrix and $\mP(\mQ(\vx,\vtheta))$ be its uniformized Markov Chain. $\mP(\mQ(\vx,\vtheta))^t$ is the Markov chain after $t$ steps where $t>0$, then the gradients of $\mP^t$ w.r.t.\ $\vtheta_k$ is
\begin{equation} \label{eq:infsum}
\begin{aligned}
    &\!\! \nabla^{(t)}_{\vtheta_k} \Puni(\mQ(\vx, \vtheta)) \equiv \frac{\partial {\Puni(\mQ(\vx, \vtheta))^{t}}}{\partial \vtheta_k} \\
    &\!\!= \sum\limits_{l = 1}^{t}{{\Puni(\mQ(\vx,\vtheta))}^{t - l}\frac{\partial {\Puni(\mQ(\vx, \vtheta))}}{\partial \vtheta_k} {\Puni(\mQ(\vx,\vtheta))}^{l - 1}}, 
\end{aligned}
\end{equation}
where 
$$
\frac{\partial {\Puni(\mQ(\vx, \vtheta))}}{\partial \vtheta_k} = 
\sum_{ij} \frac{\partial {\Puni(\mQ)}}{\partial q_{ij}} \frac{\partial q_{ij}(\vx,\vtheta)}{\partial \vtheta_k}.
$$
\end{lemma}
The proof is in the \AppProA.
Because $\sum_{l=1}^\infty \Puni(\mQ(\vx,\vtheta))^{l-1}$ diverges for any valid $\vx$ and $\vtheta$, it is not obvious that Equation~\eqref{eq:infsum} converges to a unique fixed point for $t \to \infty$.
In what follows we show that the gradient in Equation~\eqref{eq:infsum} exists and is unique as $t \to \infty$:

\begin{proposition}[Infinite Gradient Series Simplification] \label{prop:infgrad}
Let $\mQ$ be a K-state transition rate matrix of a stationary and ergodic MC. Equation~\eqref{eq:infsum} for $t \to \infty$, henceforth denoted $\nabla^{(\infty)}_\mQ \Puni(\mQ) \equiv  \lim_{t \rightarrow \infty} \nabla^{(t)}_\mQ \Puni(\mQ)$, \textbf{exists and is unique} and can be redefined as
\vspace{-5pt}
\begin{equation} \vspace{-3pt}
\label{eq:infinitySplit}
\begin{aligned}
	(\nabla^{(\infty)}_\mQ \Puni(\mQ))_{ij} &\equiv  \lim_{t \rightarrow \infty}{\sum\limits_{l = 1}^{t}{\left( \Puni^{t - l}\frac{\partial \Puni(\mQ)}{\partial q_{ij}}\Puni^{l - 1} \right)}} \\
	&= \boldsymbol{\Pi} \sum_{l = 0}^{\infty} \frac{\partial \Puni(\mQ)}{ \partial q_{ij}} \Puni^{l},
\end{aligned}
\end{equation}
where $\boldsymbol{\Pi}$ is a matrix whose rows are the steady state distribution $\bpi$.
Note that the diagonal $i=j$ is trivial to compute but should be treated as a special case.
\end{proposition}
The proof in the \AppProB shows that $\nabla^{(\infty)}_\mQ \Puni(\mQ)$ converges because the term inside the sum converges to a matrix of zeros as $l \to \infty$.
Using Proposition~\ref{prop:infgrad} it is easy to prove that Equation~\eqref{eq:infsum} converges as $t \to \infty$.

While Proposition~\ref{prop:infgrad} shows that $\nabla^{(\infty)}_\mQ \Puni(\mQ)$ converges, evaluating the infinite sum in Equation~\eqref{eq:infinitySplit} is challenging.
Truncating the sum would make the gradient biased, deviating the fixed point solution of Equation~\eqref{eq:MLE}.  
To circumvent the infinite sum in Equation~\eqref{eq:infinitySplit}, we propose $\infty$-SGD, a numerically stable stochastic gradient descent method that can optimize gradients with infinite sums ---as long as the sum is a weakly convergent series.
Our experiments show that $\infty$-SGD consistently outperforms BPTT in stability to hyperparameters in convergence rate, and in estimation accuracy.
%

\begin{theorem}[Infinity Stochastic Gradient Descent ($\infty$-SGD)] \label{thm:infSGD}
Let $\mQ$ be the transition rate matrix of a stationary and ergodic CTMC. Assume strictly positive values for the learnable parameters $\vtheta^{(h)}$ at the $h$-th step of the optimization. Let $\mP(\mQ(\vx, \vtheta^{(h)}))$ be its uniformized transition probability matrix per Definition~\ref{def:unif}.
Let $\Loss(\vy,\bpi)$ be as in Equation~\eqref{eq:loss}.
%
Reparameterize $\tilde{\Loss}(\vy, \vx, \vtheta^{(h)}) = \Loss(\vy,\bpi(\vx,\vtheta^{(h)}))$ as the loss function with respect to $\vtheta^{(h)}$.
Let $X^{(h)} \sim \text{Geometric}(p^{(h)})$, $X^{(h)} \in \sZ^+$, be an independent sample of  a Geometric distribution with $p^{(h)} < \delta^{(h)}$, where $\delta^{(h)}$ is the spectral gap of $\mP(\mQ(\vx, \vtheta^{(h)}))$.
Then, for $0 < \epsilon \ll 1$ and for all learnable parameters $\vtheta$,
\begin{align*}
	\vtheta^{(h+1)}_{k} &= \max\left(\vtheta^{(h)}_k - \eta^{(h)} \nabla_{\vtheta_k} \left. \tilde{\Loss}(\vy,\vx,\vtheta) \right\vert_{\vtheta=\vtheta^{(h)}}, \epsilon\right), 
\end{align*}
where
	%
\begin{align*} 
	\nabla_{\vtheta_k} \tilde{\Loss}(\vy,\vx,\vtheta) &= \sum\limits_{ij} \sum\limits_{mn} (\rvp^\text{(events)}(0))_m \left. \frac{\partial \Loss(\vy,\bpi)}{\partial \bpi_{n}} \right\vert_{\bpi=\bpi(\vx,\vtheta)}\\
	&\qquad \times \bpi(\vx,\vtheta)_{n} \Gamma_{ijmn}(\vx, \vtheta) \frac{\partial \mQ(\vx,\vtheta)_{ij}}{\partial \vtheta_k}
\end{align*}
with $h = 0,1,\ldots,$ where $\bpi(\vx,\vtheta)$ is the steady state distribution defined in Equation~\eqref{eq:pi}, $\eta^{(h)}$ is the learning rate with $\sum_{h=0}^\infty \eta^{(h)} = \infty$, $\sum_{h=0}^\infty \left(\eta^{(h)}\right)^2 < \infty$, and
\begin{equation} \label{eq:infGamma}
   \!\! \Gamma_{ijmn}(\vx,\vtheta)\! =\! \sum_{t = 0}^{X^{(h)}}\! \left[ \frac{\partial \mP(\mQ(\vx, \vtheta))}{\partial q_{ij}}{\frac{\mP(\mQ(\vx, \vtheta))^t}{\Prob[X^{(h)} > t]}} \right]_{mn}\!\!,
\end{equation}
is a stochastic gradient descent method that minimizes Equation~\eqref{eq:MLE}.
\end{theorem}
The proof of the theorem is in the \AppProC. 
The main insight is the use of Proposition~\ref{prop:infgrad} to produce a randomly-stopped unbiased estimator.
The requirement in Theorem~\ref{thm:infSGD} that $p^{(h)} < \delta^{(h)}$ comes from a loose bound, \ie, in practice $p^{(h)}$ can be relatively large (larger than the spectral gap) as our empirical results show ---\eg, all of our empirical results use the constant $p^{(h)}=0.1, \forall h$.  We have also tested some experiments with $p^{(h)}=0.01$, which works as well as $p^{(h)} = 0.1$ (see \AppExpD). As it is application-dependent, the value of $p^{(h)}$ should be seen as a hyperparameter.
In what follows we introduce our empirical results.

%
%
%
%
\section{Results} \label{sec:results}

In this section, we contrast the accuracy and convergence of $\infty$-SGD (Theorem~\ref{thm:infSGD}) against DC-BPTT (Definition~\ref{def:DCBPTT}) and find that $\infty$-SGD is more stable and consistently learns more accurate models.
The primary application of our experiments is predicting {\em request loss rates in a queueing system} from data that has no observed losses, under the following conditions: 
(a) we learn $\vtheta^\star$ of Equation~\eqref{eq:MLE} as a function of known request rate $\vx^\text{light} \in \Lambda^\text{light}$ under light load (no losses) in the training data, and predict $\bpi^\text{heavy}$, the steady state request loss rates under heavy loads in the test data (out-of-sample extrapolation), where $\bpi^\text{heavy}$ is such that $\left(\bpi^\text{heavy}\right)^\mathsf{T} \mQ(\vx^\text{heavy}; \vtheta^\star) = 0$ with  $\vx^\text{heavy} > \max(\Lambda^\text{light})$; 
moreover, (b) only part of the state space is observed in the training data, $\sS' \subset \sS$, and we wish to predict the steady state probability of the unobservable states $\overline{\sS}' =  \sS \backslash \sS'$.

\vspace{-10pt}
\paragraph{Baseline method.}
Due to the absence of methods on parametric inference of CTMCs from steady-state observations, 
our main baseline is the DC-BPTT of Definition~\ref{def:DCBPTT}.
In most of our simulations, we set $t^\star = 128 = 2^7$ and $\rvp^\text{(events)} (0) = {\bf 1}^\text{T}/|\sS|$ throughout all our experiments.
We also tested BPTT without divide and conquer but find the optimization unstable due to the long backpropagation paths.
We tested $t^\star \in \{16,128\}$ and found that smaller values of $t^\star$ are easier to optimize but ---as expected--- generally produce worse approximations of the steady state for heavy loads.

\vspace{-10pt}
\paragraph{Infinity learning.}
Our experiments also test our proposed approach, $\infty$-SGD, with $X^{(h)} \sim \text{Geometric}(p)$ of Theorem~\ref{thm:infSGD}, where $p$ is a constant success probability, \ie, $\expected[X] = 1/p$. 
In most of our experiments, $p=0.1$, that is, on average we consider only the first ten terms in the sum of Equation~\eqref{eq:infinitySplit}.
Contrast, $\infty$-SGD's 10 summation terms with matrix powers that need no chain rule, with the baseline DC-BPTT approach (Definition~\ref{def:DCBPTT}) where $\mP(\mQ(\vx,\vtheta))^{128}$   needs to be computed together with a chain rule to compute gradients over the matrix multiplications.
It is no surprise that $\infty$-SGD is a more stable optimization method (no vanishing or exploding gradients); interestingly, $\infty$-SGD also works well on the tested slow-mixing CTMCs, while baseline methods like DC-BPTT fail in these scenarios (see \AppExpE).

\vspace{-10pt}
\paragraph{Relaxing the parametric model.}
In some of our experiments, we will construct transition rate matrix $\mQ'(\vx, \vtheta, \tilde{\mQ})=\mQ(\vx, \vtheta) + \tilde{\mQ}$ with an $\alpha \Vert \tilde{\mQ} \Vert_2^2$ regularization penalty,  $\alpha > 0$, where $\tilde{\mQ}$ is an additional non-parametric learnable matrix s.t. $\tilde{\mQ}_{ij} = 0$ whenever $(\mQ(\vx, \vtheta))_{ij} \neq 0$, otherwise $\tilde{\mQ}_{ij}$ is a learnable parameter of our model. 
This allows some uncertainty on the form of our parametric models.
It also allows us to learn the parameters through an interpolation between parametric and non-parametric CTMC models.

The regularization term $\alpha \Vert \tilde{\mQ} \Vert_2^2$ is added to the negative log-likelihood loss in Equation~\eqref{eq:MLE} to ensure that we can control how much flexibility we want.
With small values of $\alpha$, we are testing how overparameterization, \ie, having too many extra parameters in $\mQ'$, affects learning and generalization.
Our experiments show that $\alpha \gg 1$ gives the best results, \ie, the correct parametric model works best. 
We also see that $\alpha \approx 1$ still gives competitive results (refer to \AppExpC), showing that some model flexibility is tolerable.
In contrast, we see that $\alpha = 0.1$ tends to significantly hurt our ability to extrapolate queue losses in the test data.

\begin{table*}
\vspace{-5pt}
\caption{\small [MAPE] Simulation results showing MAPE/100 ($\langle$Mean Absolute Error$\rangle$/$\langle$true value$\rangle$) errors between predicted steady state and ground-truth for failure states in test data (heavy load). Training data collected under light loads and restricted observed states (queues zero and one).
Mixing rates are determined by the spectral gaps $\delta_n$ observed in training data over multiple time windows ($n = 1,\ldots, 50)$. With 95\% confidence intervals.}
\label{tab:mape}
\vspace{-2pt}
\centering
\resizebox{1.\textwidth}{!}{
\begin{tabular}{@{\extracolsep{0pt}}lclll@{}}
& $\delta_n$ (spectral gap) & \multicolumn{1}{c}{DC-BPTT $t^\star = 16$} & \multicolumn{1}{c}{DC-BPTT $t^\star = 128$} & \multicolumn{1}{c}{$\infty$-SGD ($p=0.1$)} \\
\cmidrule(l){2-5}

Testbed Emulation (Upper Trig.)   & N/A            & $ 1.43 \times 10^{1} \pm 0.00 ~~~~~~~~~~~ $ & $ 1.88 \times 10^{1} \pm 0.00 ~~~~~~~~~~~ $ & $ \mathbf{ 9.33 \times 10^{-1} \pm 8.91 \times 10^{-2} } $ \\
M/M/1/$K$ (fast-mix)                   & [0.022, 0.043] & $ 2.04 \times 10^{-1} \pm 2.86 \times 10^{-4} $ & $ 1.32 ~~~~~~~~~~~ \pm 4.94 \times 10^{-3} $ & $ \mathbf{ 1.18 \times 10^{-2} \pm 8.01 \times 10^{-3} } $ \\M/M/1/$K$ (slow-mix)                 & [0.005, 0.008] & $ 8.91 \times 10^{3} \pm 2.02 \times 10^{2} $ & $ 6.68 \times 10^{1} \pm 7.61 ~~~~~~~~~~~ $ & $ \mathbf{ 8.88 \times 10^{-1} \pm 1.48 ~~~~~~~~~~~ } $ \\
M/M/$m$/$m+r$                          & [0.013, 0.024] & $ 4.11 \times 10^{-1} \pm 4.47 \times 10^{-2} $ & $ 4.01 \times 10^{-1} \pm 8.86 \times 10^{-2} $ & $ \mathbf{ 1.52 \times 10^{-1} \pm 8.50 \times 10^{-2} } $ \\
M/M/Multiple/$K$                       & [0.068, 0.096] & $ 9.09 \times 10^{-1} \pm 1.20 \times 10^{-2} $ & $ 4.03 \times 10^{1} \pm 6.74 \times 10^{-2} $ & $ \mathbf{ 2.27 \times 10^{-1} \pm 1.47 \times 10^{-2} } $ \\
\hline
\end{tabular}
}
\end{table*}

\begin{table*}
\vspace{-2pt}
\caption{\small [MSE] Simulation results showing MSE errors between predicted steady state and ground-truth for failure states in test data (heavy load). Training data collected under light loads and restricted observed states (queues zero and one).
Mixing rates are determined by the spectral gaps $\delta_n$ observed in training data over multiple time windows ($n = 1,\ldots, 50)$. With 95\% confidence intervals.}
\vspace{-5pt}
\label{tab:mse}
\centering
\resizebox{1.\textwidth}{!}{
\begin{tabular}{@{\extracolsep{0pt}}lclll@{}}
& $\delta_n$ (spectral gap) & \multicolumn{1}{c}{DC-BPTT $t^\star = 16$} & \multicolumn{1}{c}{DC-BPTT $t^\star = 128$} & \multicolumn{1}{c}{$\infty$-SGD ($p=0.1$)} \\
\cmidrule(l){2-5}
Testbed Emulation (Upper Trig.) & N/A            & $ 4.80 \times 10^{-1} \pm 0.00 ~~~~~~~~~~~ $ & $ 8.45 \times 10^{-1} \pm 0.00 ~~~~~~~~~~~ $ & $ \mathbf{ 2.41 \times 10^{-3} \pm 5.20 \times 10^{-4} } $ \\
M/M/1/$K$ (fast-mix)                 & [0.022, 0.043] & $ 6.81 \times 10^{-3} \pm 1.99 \times 10^{-5} $ & $ 2.45 \times 10^{-1} \pm 1.84 \times 10^{-3} $ & $ \mathbf{ 4.98 \times 10^{-5} \pm 5.37 \times 10^{-5} } $ \\
M/M/1/$K$ (slow-mix)                 & [0.005, 0.008] & $ 1.14 \times 10^{-2} \pm 1.54 \times 10^{-4} $ & $ \mathbf{ 4.84 \times 10^{-4} \pm 1.99 \times 10^{-4} } $ & $ \mathbf{ 9.36 \times 10^{-4} \pm 1.58 \times 10^{-3} } $ \\
M/M/$m$/$m+r$                        & [0.013, 0.024] & $ 4.25 \times 10^{-2} \pm 7.25 \times 10^{-3} $ & $ 2.87 \times 10^{-2} \pm 1.09 \times 10^{-2} $ & $ \mathbf{ 6.65 \times 10^{-3} \pm 5.61 \times 10^{-3} } $ \\
M/M/Multiple/$K$                     & [0.068, 0.096] & $ 8.97 \times 10^{-3} \pm 1.24 \times 10^{-4} $ & $ 6.84 \times 10^{-1} \pm 1.99 \times 10^{-3} $ & $ \mathbf{ 5.34 \times 10^{-4} \pm 8.23 \times 10^{-5} } $ \\
\hline
\end{tabular}
}
\end{table*}

\subsection{Testbed Experiments} \label{subsec:testbed_exp}
%
We now contrast DC-BPTT against $\infty$-SGD in a real-world testbed emulating a Voice-over-LTE (VoLTE) system in a wireless cellular network.
The testbed is configured as a single server with a waiting queue of size $K=20$.
The training data (86 time windows) is generated under light loads (with mean 7.7 and median 3 call losses) and the test data (137 time windows) under heavy loads (with mean 135.2 and median 254 call losses).
{\em Moreover, we also restrict the observations in the training data, $\sS' \subset \sS$, to queue sizes one and two, estimated from the request processing delays collected at the clients.}
We define $\mQ(\vx; \vtheta)$ symbolically using Pytorch's autodiff function~\citep{pytorch}.
The \AppExpA contains the details of our experimental setup and methodology.
We study two parametric $\mQ(\vx,\vtheta)$ models:

\emph{(A) M/M/1/$K$ model:} 
We start with arguably the most fundamental CTMC parametric model of a queueing system, the M/M/1/$K$ queue with a single server and a single queue of size $K=21$, which can hold up to $20$ requests waiting for service.
The queue can be described by a CTMC $\mQ(x_m,\theta^\star)$, where
service requests arrive according to a Poisson process with request rate $x_m$ assumed constant over time window $m$ ---the length of a time window is the minimum time resolution of our logs (one second). 
If a new request arrives when the queue is full, it is dropped by the server.
Service times are exponentially distributed with rate $\theta^\star$, assumed constant ---the service rate capacity of the system.
We assume $K$ is known.
The request call rate $x_m$ at time window $n$ is known while $\theta$ is the only parameter that needs to be learned.
The system is assumed in steady state even over short time windows.

\emph{(B) Upper Triangular model:}
Since it is difficult to fit a real system with an exact CTMC queueing model, we consider an embedded birth-death process $\mQ(x_m; \vtheta^\star)$, called \emph{Upper Triangular model}, which sets the upper triangular portion of $\mQ$ to $x_m$ in the upper diagonal and zeros everywhere else. 
This model (setting the upper triangular to zeros) indicates that only a client request can increase the queue size.
The lower triangular part of $\mQ$ is populated with different parameters that we will learn, which implies learning $\vert \vtheta^\star \vert=|\sS|(|\sS|-1)/2$ parameters for a CTMC with $|\sS|$ states.
%
%
The learnable parameters $\vtheta$ are initialized with zeros and learned with either DC-BPTT or $\infty$-SGD.
We also know the maximum queue size $K=21$ as in the M/M/1/$K$ scenario.
Since we want to predict overload, we add a new state to represent that case.
Any arrival after reaching the last state, when the current queue size is $K$, will force the system to transit to this overload state. Thus, we have $|\sS| = K + 2$.

\begin{figure*}
\begin{minipage}{.3\textwidth}
\centering
\includegraphics[width=1.8in,height=1.2in]{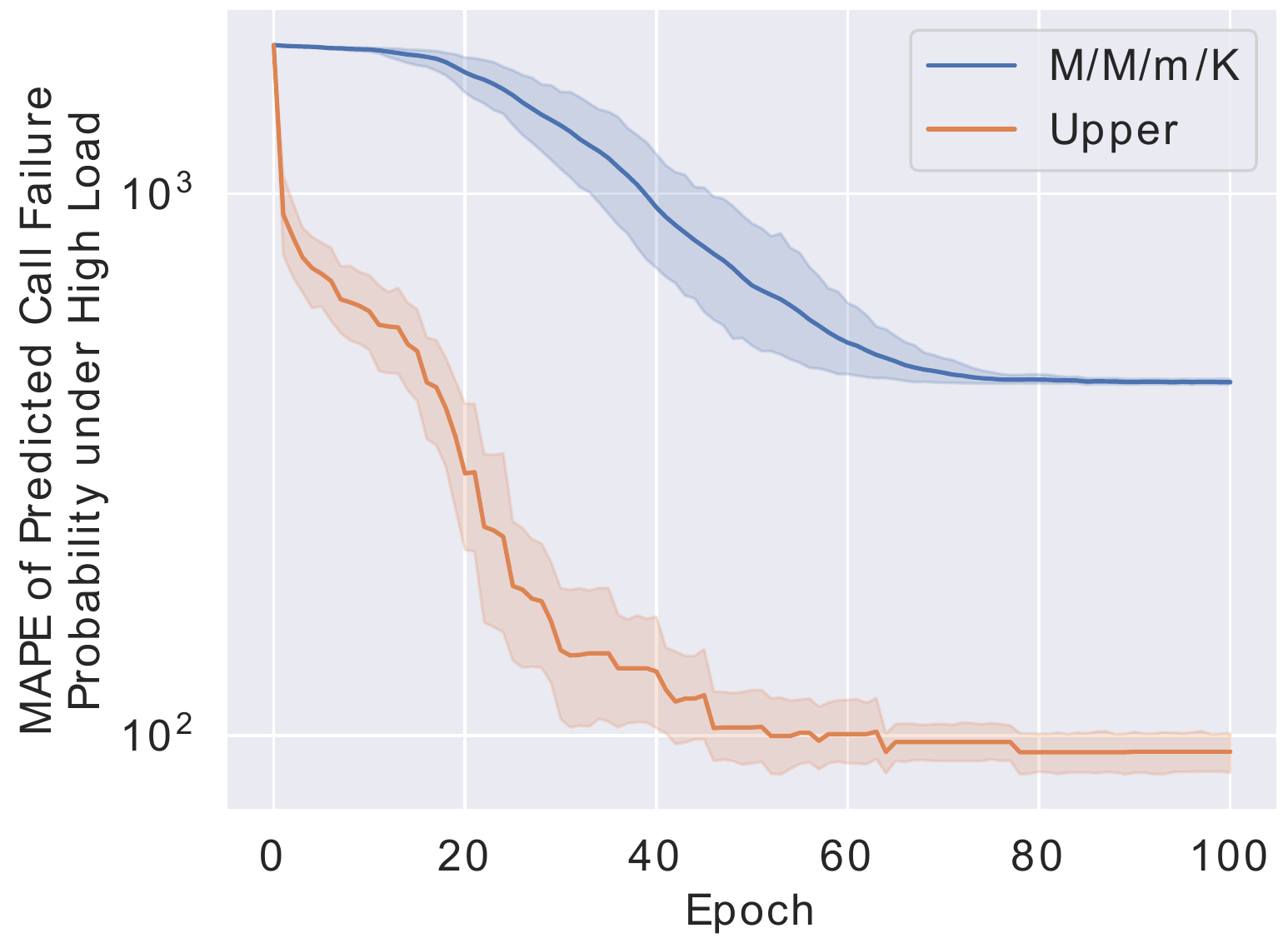}
\vspace{-4pt}
\subcaption{(Training curves) Effect of parametric models on test MAPE (learned with $\infty$-SGD).~\\}
\label{fig:rtt-alpha}
\end{minipage}
\hfill
\begin{minipage}{.3\textwidth}
\vspace{-7pt}
\centering
\includegraphics[width=1.8in,height=1.3in]{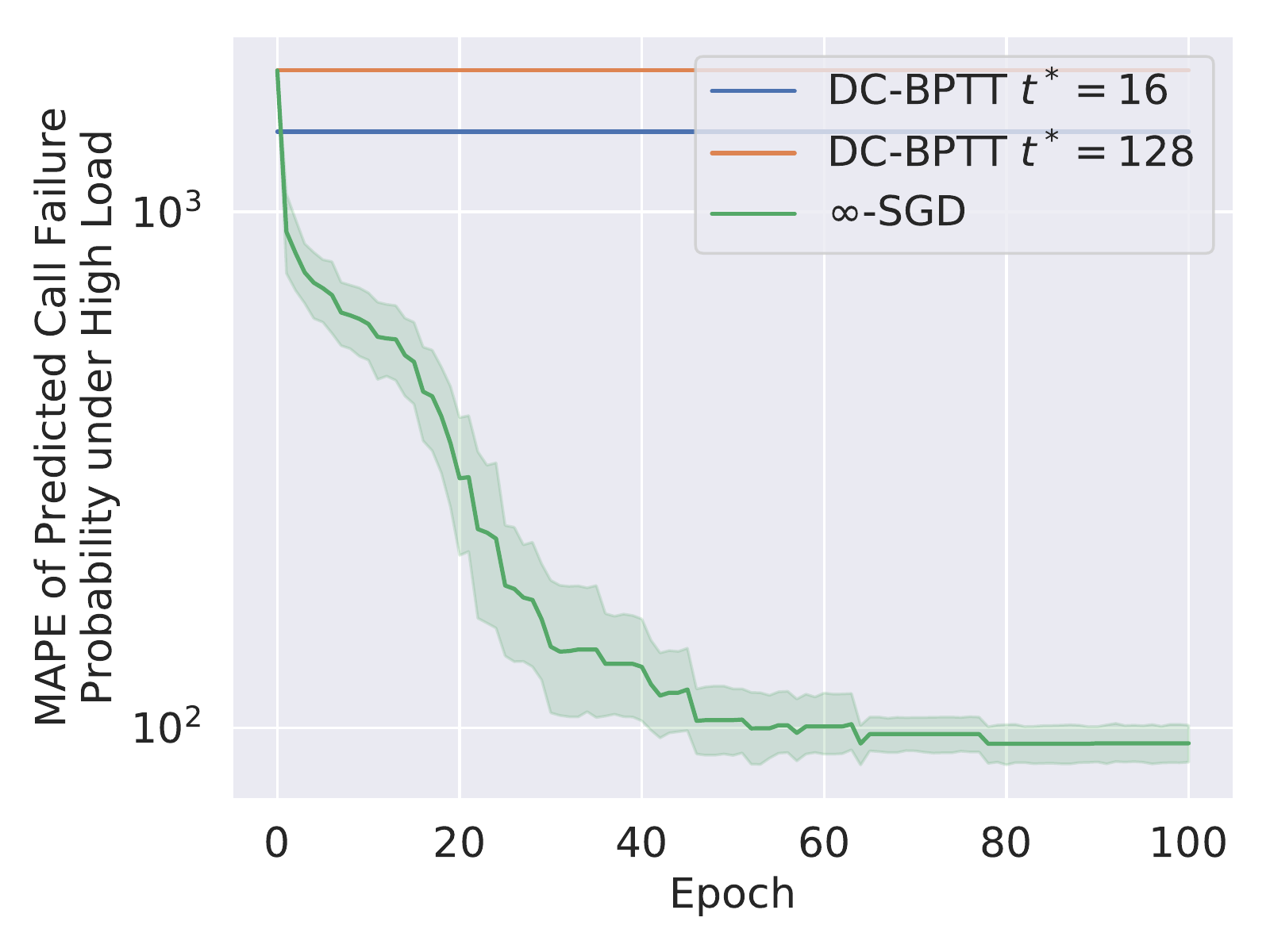}
\vspace{-6pt}
\subcaption{(Training curves) Effect of learning methods on test MAPE of upper triangular model.}
\label{fig:ablat-rtt-magic}
\end{minipage}
\hfill
\begin{minipage}{.33\textwidth}
\centering
\vspace{-10pt}
\includegraphics[width=1.9in,height=1.3in]{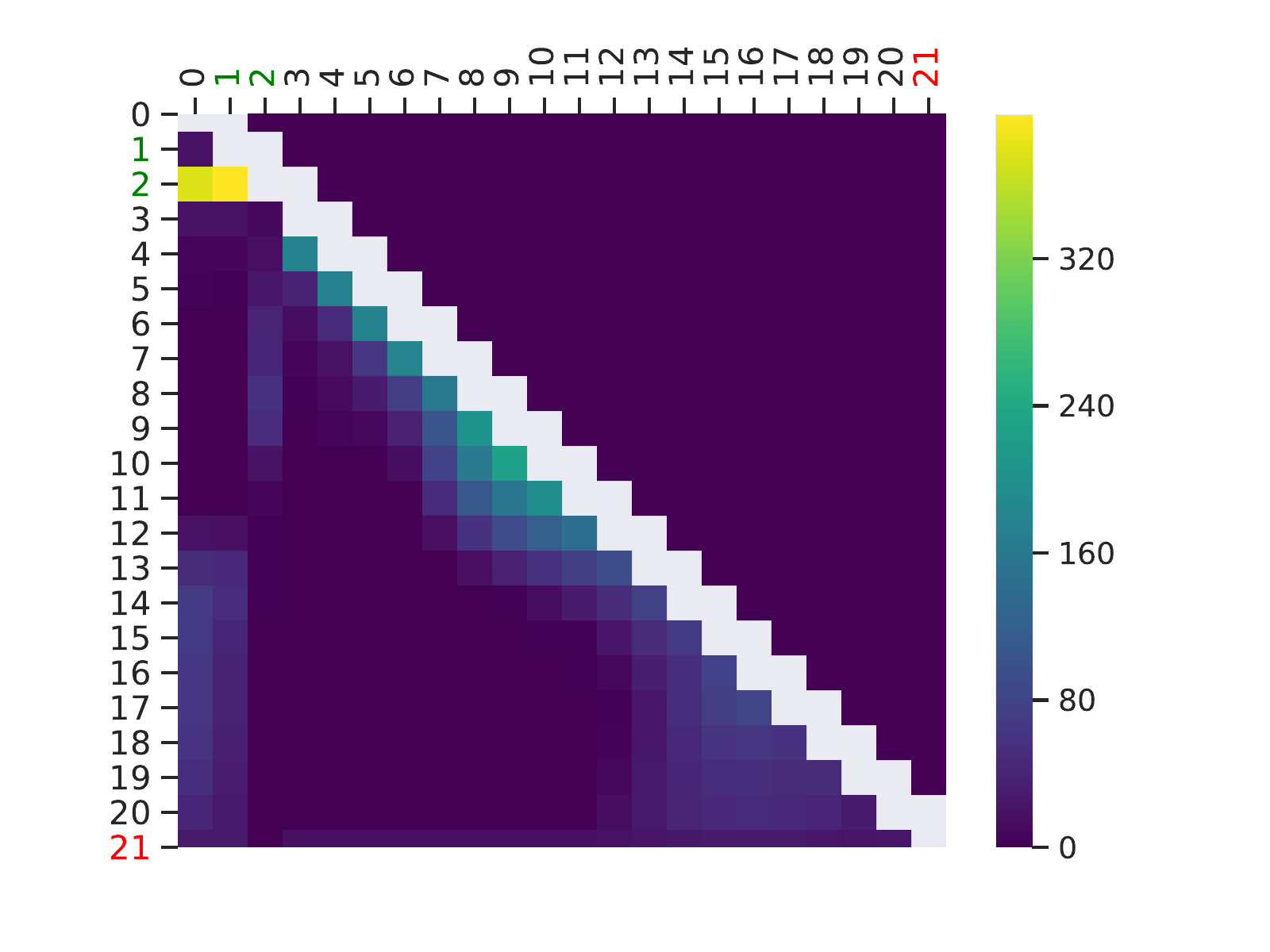}
\vspace{-12pt}
\subcaption{$\mQ(\vx,\vtheta^\star)$ learned by {\em upper triangular} model with $\infty$-SGD. Upper triangular ($\vx$) is removed or shows zeros, lower triangular shows $\vtheta^\star$.}
\label{fig:rtt-mx}
\end{minipage} 
\vspace{-9pt}
\caption{\small Real-world experiment results on VoLTE testbed. (a-b) Test error (MAPE) of unseeing failure state ---here, {\em ground-truth} of dropped call probability--- under heavy load (while training under light loads), as a function of training epochs. In these plots we verify the better generalization and stability of $\infty$-SGD. (a) Shows that more flexible {\em Upper Triangular} parametric model has much smaller (near-zero) test error than the more strict M/M/1/$K$ parametric model.
(b) Shows that $\infty$-SGD significantly outperforms DC-BPTT (which fails to learn). (c) Learned $\mQ(\vx,\vtheta^\star)$ by the upper triangular parametric model with $\infty$-SGD, showing an emergent block structure.}
\label{fig:exp_testbed_results}
\vspace{-10pt} 
\end{figure*}

\vspace{-8pt}
\paragraph{Results.}
\Figref{fig:rtt-alpha} shows the mean absolute percentage error (MAPE) between the predicted call drop probability (given the call request rate) and the true call drop probability in the test data under heavy loads, for the upper triangular and M/M/1/$K$ parametric model learned with $\infty$-SGD.
The {\em Upper Triangular} model achieves much lower test mean squared error (MSE) ($ 2.57 \times 10^{-3} \pm 5.75 \times 10^{-4} $) corresponding to MAPE of about 80\% to 100\% over test call drop probabilities in the range [0.0045, 0.1629], \ie, it more accurately extrapolates the training data (light loads, just observing queues of size one and two) to the test (heavy loads, full queue).
The M/M/1/K parametric model is too simple and performs poorly with test MAPE of 449\% (test MSE is reasonable at $7.78\times 10^{-2}$).

We also investigate the transition rate matrix learned by upper triangular model in \Figref{fig:rtt-mx}. We note that the learned queue is quite similar to an M/M/1/$K$, but the service rate is decreasing as the queue size is increasing. 
Surprisingly, this is a real phenomenon when real systems start to become overloaded~\citep{Jain88}. 
We also see some {\em reset} transitions, where the system goes from a full queue to a nearly empty queue.
Finally, \Figref{fig:ablat-rtt-magic} shows that only $\infty$-SGD can learn the $\mQ$ of the upper triangular model ---which has $|\sS|(|\sS|-1)/2=210$ parameters---, while DC-BPTT has vanishing gradients for both $t^\star \in \{16,128\}$ ---we note that $t^\star=16$ has a smaller loss than $t^\star=128$.
Finally, Tables~\ref{tab:mape} and~\ref{tab:mse} reaches the obvious conclusion that $\infty$-SGD learns significantly better models for extrapolation over the test data than DC-BPTT.

In what follows we explore the differences between DC-BPTT and $\infty$-SGD in synthetic experiments.

\subsection{Synthetic Experiments}

We now turn our attention to simulations.
Due to space limitations, we give a succinct description of the experiments, relegating details and additional results to \Appendix{}s C2, C3, C4, and C5.

{\em Birth-death queues:} We start with arguably the most fundamental parametric CTMC queueing system, the diagonal structure of the \emph{birth-death process}. The birth-death process approximates a number of queueing systems, such as the M/M/1/$K$ queue with a single queue of size $K$ and a single server, and the M/M/m/$K$ queue with $m$ servers that has been used to approximate cloud services~\citep{MMmm+r,MGmm+r}.
Since queue size $K$ must be larger than $m$ to support all servers, the M/M/m/$K$ queue is typically denoted as M/M/m/$m+r$, where $r \geq 0$.

\vspace{-9pt} 
\paragraph{Training and testing data.}
Structures used to simulate data are provided in \AppEgC.
We pre-define the service rates ($\theta^\star=25$ for M/M/1/K (slow-mix and fast-mix), $\theta^\star=5$ for M/M/5/5+r, $\vtheta^\star = (15,10,5)$ for M/M/Multiple/K) and queue sizes $K=20$ ($r=15$), then at each time window in the training data we sample a request rate uniformly in the interval $\vx \in [11,15]$ (light load), except for M/M/1/K slow-mix $\vx \in [21,30]$ (to decrease the spectral gap).
At test time, in the test data, we sample a request rate uniformly in the interval $\vx \in [31,60]$ (heavy load), except for M/M/1/K slow-mix $\vx \in [11,40]$ (to decrease the spectral gap).
%
We assume we {\em only} observe the queue size if it is empty or it has exactly one request, \ie, $\sS' = \{0,1\}$.
This emulates a common trend in logging critical infrastructure systems, where logging stops as soon as the server load is non-trivial~\citep{logperf}.
We also have extra results with different transition rates (in an easier task where $\infty$-SGD does even better) and more details on our training and test data generation in \Appendix{}s C2 and C5.

\vspace{-10pt} 
\paragraph{Results.}
Our goal is to predict the request loss probability against ground-truth under a range of both heavy and lighter loads, while training under a narrow range of light loads.
In the M/M/1/$K$ and M/M/m/m+r simulations, the training data consists of the aggregate frequencies observed for queue sizes zero and one during one second, along with the request rate. For M/M/Mutiple/$K$, aggregate frequencies for queue sizes zero to three are observed.

Tables~\ref{tab:mape} and~\ref{tab:mse} compare the extrapolation error of DC-BPTT and $\infty$-SGD in our synthetic experiment using MAPE and MSE errors, respectively.
Our approach, $\infty$-SGD, is consistently better than DC-BPTT over all simulations and on both error metrics (MAPE and MSE).
For a slow-mixing M/M/1/$K$, $\infty$-SGD extrapolation MAPE error is $1/100$-th of DC-BPTT MAPE error, considering the confidence interval.
In some of the scenarios, DC-BPTT finds gradient vanishing problems (failing to learn) giving very large errors (see training curves in \AppExpE), while $\infty$-SGD never fails to obtain gradients that can reduce the loss during the optimization.

MAPE result shows that $t^\star \leq 128$ is not enough to see the slow-mixing chain in steady state. Success in MSE for slow mixing while failing in MAPE shows that DC-BPTT has trouble learning parametric CTMCs well enough to predict out-of-sample (extrapolated) target states that have small probabilities. 
Moreover, DC-BPTT $t^\star=16$ tends to achieve lower errors (both MAPE and MSE) than DC-BPTT $t^\star=128$ in the M/M/1/$K$ fast mixing scenarios.

We now look at ground-truth $\theta^\star$ parameters and their estimates $\hat{\theta}$ from $\infty$-SGD. 
We note that the estimates are very close to the true values. In the M/M/1/$K$ model (true $\theta^\star=25$), the (slow mixing) scenario gives $\hat{\theta}=25.003$, and (fast mixing) gives $\hat{\theta}=25.083$. For M/M/m/m+r (with true $\theta^\star=5$) obtains $\hat{\theta}=5.15$. For M/M/Multiple/$K$, $\infty$-SGD obtains $\hat{\vtheta}=(13.5, 8.3, 5.4)$, close to the ground truth $\vtheta^\star=(15, 10, 5)$.
This conclusively shows $\infty$-SGD to be a reliable optimization method.

%
%
%
%
\vspace{-4pt} 
\section{Conclusions}
\vspace{-2pt} 
This work introduces $\infty$-SGD, the first theoretically principled optimization approach that can accurately learn general {\em parametric} Continuous Time Markov Chains (CTMCs) from aggregate steady-state observations. 
Our approach, $\infty$-SGD,  works even when the observations are over a restricted set of states.
We have shown that $\infty$-SGD finds significantly better maximum likelihood estimates than the baseline (DC-BPTT) in both a real testbed and synthetic scenarios.
Moreover, in the context of queueing systems, $\infty$-SGD consistently better extrapolates from training data in light loads to heavy loads in test data. We expect $\infty$-SGD to be a useful tool in other tasks where parametric models are needed and sequence data is only available as aggregate frequencies.

%
%
%
%
\section{Acknowledgement}
This work has been sponsored in part by the ARO, under the U.S. Army Research Laboratory contract number W911NF-09-2-0053, the Purdue Integrative Data Science Initiative, and the National Science Foundation grants CNS-1717493, OAC-1738981, and CCF-1918483.

{
\small
\bibliography{references.bib}
\bibliographystyle{aaai}
}

\onecolumn
\section*{\Appendix}
%
%
%
%
\subsection*{\AppEgA: Unidentifiability} \label{subsec:app:ega}

Let us focus on the following two transition matrices of different Markov chains:
$$
	\mP = \begin{bmatrix}
	0.7 & 0.3 & 0   & 0   \\
	0.4 & 0.4 & 0.2 & 0   \\
	0   & 0.3 & 0.6 & 0.1 \\
	0   & 0   & 0.2 & 0.8 \\
	\end{bmatrix}
	\quad
	\mP' = \begin{bmatrix}
	0.7 & 0.3 & 0   & 0   \\
	0.4 & 0.5 & 0.1 & 0   \\
	0   & 0.3 & 0.5 & 0.2 \\
	0   & 0   & 0.1 & 0.9 \\
	\end{bmatrix}
$$
Their steady state distribution can be achieved directly:
$$
\begin{aligned}
	\mP^{\infty} &= \begin{bmatrix}
	0.4 & 0.3 & 0.2 & 0.1 \\
	\end{bmatrix} \\
	{\mP'}^{\infty} &= \begin{bmatrix}
	0.4 & 0.3 & 0.1 & 0.2 \\
	\end{bmatrix} \\ 
\end{aligned}
$$
Suppose in our learning task, a perfect aggregate frequencies collection is done for state 1 and state 2.
In other words, we know the ground truth steady state distribution of state 1 and state 2, which are 0.4 and 0.3.
If we regard state 4 as the failure state we want to extrapolate, we can easily notice that it is unidentifiable: $A_{1}$ and $A_{2}$ can both match the aggregate frequencies, while they have totally different distribution on state 4.

%
%
%
%
\subsection*{\AppEgB: A Slow Mixing Event} \label{subsec:app:egb}

We give an example on slow mixing M/M/1/$K$. Suppose we have following $20 \times 20$ matrix
$$
	\mQ = \begin{bmatrix}
	-25 & 25 \\
	24 & -49 & 25 \\
	& 24 & -49 & 25 \\
	& & \ddots & \ddots & \ddots \\
	& & & 24 & -49 & 25 \\
	& & & & 24 & -24 \\
	\end{bmatrix}.
$$
We can then get a stochastic matrix $\mP$ by uniformizing $\mQ$ as Definition~\ref{def:unif}. We say a matrix is fast on mixing if all rows of $\mP^{t}$ are similar with small $t$, and vise versa.

We define
\begin{equation}
	\mix{\mP, \epsilon} = \min \left\{t \in \sN+; \max\limits_{i = 1, \cdots, |\sS|} \left({ \sum\limits_{j = 1}^{|\sS|} {\left\Vert \mP^{t}_{ij} - \frac{1}{|\sS|} \sum\limits_{k = 1}^{|\sS|}{\mP^{t}_{kj}} \right\Vert_2^2} }\right) \leq \epsilon \right\}
\end{equation}
as the criterion for mixing speed, where $\vert \sS \vert$ is the number of states in the CTMC. This represents the minimum exponent $t$ for $\mP^{t}$ to reaching a mixing status defined by $\epsilon$.
The larger value of $t$, the slower speed for $\mP$ to mix.
For instance, we will have $\mix{\mP, 10^{-5}} = 280$.
This means that for the defined $\mP$, we need $\mP^{280}$ to reach mixing with our $\epsilon$ tolerance (which we found unyielding for the autodiff library of Pytorch~\citep{pytorch}).
In contrast, there is no backpropagation in $\infty$-SGD, allowing it to easily deal with slow-mixing CTMCs.

If we replace in the lower diagonal line of $\mQ$ the value 24 by 1, and fix related diagonal line from -49 to -26, we can get a new transition matrix $\mP'$.
Then, $\mix{\mP', 10^{-5}} = 27$ as a fast mixing example, and now backpropagation over $P^{27}$ is supported by autodiff library.

%
%
%
%
\subsection*{\AppEgC: Parametric Models} \label{subsec:app:struct}

\paragraph{M/M/m/$K$ Model.}
We illustrate M/M/m/$K$ model (apply for M/M/1/$K$ and M/M/m/m+r) in \Figref{fig:struct-mmmk} for better description of details. In \Figref{fig:struct-mmmk} green states are observed states $\sS'$, and red state is the state we will test on.

In all simulation experiments, we set $K = 20$ for \Figref{fig:struct-mmmk}. For M/M/1/$K$ simulation, $m=1$; and for M/M/m/m+r simulation, $m=5$, $r=15$ and $K=m+r=20$. Thus, suppose we denote request rate as $\vx=(x)$ and service rate as $\vtheta=(\theta)$, we will have transition rate matrix
\begin{equation}
\mQ(\vx; \vtheta) = \begin{bmatrix}
-x & x \\
\theta & -(\theta + x) & x \\
& 2\theta & -(2\theta + x) & x \\
& & \ddots & \ddots & \ddots \\
& & & m\theta & -(m\theta + x) & x \\
& & & & m\theta & -(m\theta + x) & x \\
\melepad\  & \melepad\  & \melepad\  & \melepad\  & \melepad\  & \melepad\ddots &  \melepad\ddots &  \melepad\ddots \\
\end{bmatrix}
\end{equation}

\paragraph{Upper Triangular Model.}
Upper triangular model is nearly the same as M/M/m/$K$, except that all lower triangular part are learnable parameters $\vtheta=\{\theta_{ij}\}_{i,j \in \sS, i > j}$. We use matrix version $\vtheta$ for the ease notation, and it is easy to flatten it back into vector version $\vtheta$.
\begin{equation}
\mQ(\vx; \vtheta) = \begin{bmatrix}
-x & x \\
\theta_{21} & -(\theta_{21} + x) & x \\
\vdots & \vdots & \vdots & \ddots \\
\theta_{i1} & \cdots & \cdots & -(\sum\limits_{j=1}^{i - 1} \theta_{ij} + x) & x \\
\melepad\vdots & \melepad\vdots & \melepad\vdots & \melepad\vdots & \melepad\vdots & \melepad\ddots \\
\end{bmatrix}
\end{equation}

\paragraph{M/M/Multiple/$K$ Model.}
M/M/Multiple/$K$ model stands between M/M/m/$K$ and upper triangular structures. It augments M/M/1/$K$ structure by making lower 1 to $d$ diagonal lines learanble where we pick $d=3$ in our experiments. Thus, we will have $\vtheta = (\theta_1, \theta_2, \theta_3)$.
\begin{equation}
\mQ(\vx; \vtheta) = \begin{bmatrix}
-x & x \\
\theta_{1} & -(\theta_{1} + x) & x \\
\theta_{2} & \theta_{1} & -(\sum\limits_{j=1}^{2} \theta_{i} + x) & x \\
\theta_{3} & \theta_{2} & \theta_{1} & -(\sum\limits_{j=1}^{3} \theta_{i} + x) & x \\
0      & \theta_{3} & \theta_{2} & \theta_{1} & -(\sum\limits_{j=1}^{3} \theta_{i} + x) & x \\
\melepad\vdots & \melepad\vdots & \melepad\vdots & \melepad\vdots & \melepad\vdots & \melepad\vdots & \melepad\ddots \\
\end{bmatrix}
\end{equation}


\begin{figure}
\begin{minipage}{\textwidth}
\centering
\begin{tikzpicture}[->,>=stealth',shorten >=1pt,auto,node distance=2cm,
                    semithick]
  \tikzstyle{every state}=[fill=white,draw=black,text=black]
  \tikzstyle{train}=[circle, fill=white,draw=green,text=black]
  \tikzstyle{test}=[circle, fill=white,draw=red,text=black]
  \tikzstyle{dots}=[circle, fill=white,draw=white,text=black]

  \node[train] (0)                  {$0$};
  \node[train] (1)   [right of=0]   {$1$};
  \node[state] (2)   [right of=1]   {$2$};
  \node[dots]  (inc) [right of=2]   {$\cdots$};
  \node[state] (m)   [right of=inc] {$m$};
  \node[state] (m+1) [right of=m]   {$m + 1$};
  \node[dots]  (to)  [right of=m+1] {$\cdots$};
  \node[test]  (K)   [right of=to]  {$K$};

  \path (0)    edge[bend left] node {$x$}       (1)
        (1)    edge[bend left] node {$x$}       (2)
        (1)    edge[bend left] node {$\theta$}  (0)
        (2)    edge[bend left] node {$x$}       (inc)
        (2)    edge[bend left] node {$2\theta$} (1)
        (inc)  edge[bend left] node {$x$}       (m)
        (inc)  edge[bend left] node {$3\theta$} (2)
        (m)    edge[bend left] node {$x$}       (m+1)
        (m)    edge[bend left] node {$m\theta$} (inc)
        (m+1)  edge[bend left] node {$x$}       (to)
        (m+1)  edge[bend left] node {$m\theta$} (m)
        (to)   edge[bend left] node {$x$}       (K)
        (to)   edge[bend left] node {$m\theta$} (m+1)
        (K)    edge[bend left] node {$m\theta$} (to)
        ;
\end{tikzpicture}
\end{minipage}
\caption{M/M/m/$K$ system. Transition states are defined by $i$ where $i$ is the number of requests in the queue. States in green show observed states in the training data (0,1), and state in red shows state we need to extrapolate in the test data (K) [better visualized in color].}
\label{fig:struct-mmmk}
\end{figure}

%
%
%
%
\subsection*{\AppProA: Proof of Lemma~\ref{lem:deriv}} \label{subsec:app:proa}

We restate the lemma for completeness.
\begin{lemma*} 
Let $\mQ(\vx,\vtheta)$ be a $K$-state transition matrix and $\mP(\mQ(\vx,\vtheta))$ be its uniformized Markov Chain. $\mP(\mQ(\vx,\vtheta))^t$ is the Markov chain after $t$ steps where $t>0$, then the gradients of $\mP^t$ w.r.t.\ $\vtheta_k$ is
\begin{equation*}
\begin{aligned}
    &\!\! \nabla^{(t)}_{\vtheta_k} \Puni(\mQ(\vx, \vtheta)) \equiv \frac{\partial {\Puni(\mQ(\vx, \vtheta))^{t}}}{\partial \vtheta_k} \\
    &\!\!= \sum\limits_{l = 1}^{t}{{\Puni(\mQ(\vx,\vtheta))}^{t - l}\frac{\partial {\Puni(\mQ(\vx, \vtheta))}}{\partial \vtheta_k} {\Puni(\mQ(\vx,\vtheta))}^{l - 1}}, 
\end{aligned}
\end{equation*}
where 
$
\frac{\partial {\Puni(\mQ(\vx, \vtheta))}}{\partial \vtheta_k} = 
\sum_{ij} \frac{\partial {\Puni(\mQ)}}{\partial q_{ij}} \frac{\partial q_{ij}(\vx,\vtheta)}{\partial \vtheta_k}.
$
\end{lemma*}

\begin{proof}
In defining matrix derivatives, Example 2.1 in \citet{Bhatia1994} uses the binomial expansion to define this derivative, which can be alternatively  obtained by applying the chain rule recursively:

\begin{equation*}
\begin{aligned}
	\frac{\partial \mP(\mQ(\vx, \vtheta))^{t}}{\partial \vtheta_k}
	&= \frac{\partial \mP(\mQ(\vx, \vtheta))^{t - 1}}{\partial \vtheta_k}\mP + \mP^{t - 1} \frac{\partial \mP(\mQ(\vx, \vtheta))}{\partial \vtheta_k} \\
    &= \left(\frac{\partial \mP(\mQ(\vx, \vtheta))^{t - 2}}{\partial \vtheta_k}\mP + \mP^{t - 2}\frac{\partial \mP(\mQ(\vx, \vtheta))}{\partial \vtheta_k}\right)\mP + \mP^{t - 1}\frac{\partial \mP(\mQ(\vx, \vtheta))}{\partial \vtheta_k} \\
    &= \left(\frac{\partial \mP(\mQ(\vx, \vtheta))^{t - 2}}{\partial \vtheta_k}\right)\mP^{2} + \mP^{t - 2}\frac{\partial \mP(\mQ(\vx, \vtheta))}{\partial \vtheta_k}\mP + \mP^{t - 1}\frac{\partial \mP(\mQ(\vx, \vtheta))}{\partial \vtheta_k} \\
    &= \cdots \\
    &= \sum\limits_{l = 1}^{t}{\mP^{t - l}\frac{\partial \mP(\mQ(\vx, \vtheta))}{\partial \vtheta_k}\mP^{l - 1}}. \\
\end{aligned}
\end{equation*}

\end{proof}

%
%
%
%
\subsection*{\AppProB: Proof of Proposition~\ref{prop:infgrad}} \label{subsec:app:prob}

We restate the proposition for the sake of completeness.
\begin{proposition*}
Let $\mQ$ be a K-state transition rate matrix of a stationary and ergodic MC. Equation~\eqref{eq:infsum} for $t \to \infty$, henceforth denoted $\nabla^{(\infty)}_\mQ \Puni(\mQ) \equiv  \lim_{t \rightarrow \infty} \nabla^{(t)}_\mQ \Puni(\mQ)$, \textbf{exists and is unique} and can be redefined as
\vspace{-5pt}
\begin{equation*} \vspace{-3pt}
\begin{aligned}
	(\nabla^{(\infty)}_\mQ \Puni(\mQ))_{ij} &\equiv  \lim_{t \rightarrow \infty}{\sum\limits_{l = 1}^{t}{\left( \Puni^{t - l}\frac{\partial \Puni(\mQ)}{\partial q_{ij}}\Puni^{l - 1} \right)}} \\
	&= \boldsymbol{\Pi} \sum_{l = 0}^{\infty} \frac{\partial \Puni(\mQ)}{ \partial q_{ij}} \Puni^{l},
\end{aligned}
\end{equation*}
where $\boldsymbol{\Pi}$ is a matrix whose rows are the steady state distribution $\bpi$.
Note that the diagonal $i=j$ is trivial to compute but should be treated as a special case.
\end{proposition*}

\begin{proof}
Let 
\begin{equation*}
    \mQ =\left[ \begin{array}{cccc}
        -q_{11} & q_{12} & \cdots & q_{1|\sS|} \\
        \vdots & \vdots & \cdots & \vdots \\
        q_{|\sS|1} & q_{|\sS|2} & \cdots & -q_{|\sS||\sS|}
    \end{array}
    \right]
\end{equation*}
be the rate transition matrix of the CTMC.
Define $\mP(\mQ)$ as in Definition~\ref{def:unif}.
Then,
\begin{align*}
    \left( \frac{\partial \mP(\mQ)}{\partial q_{ij}} \right)_{kh}
    &= \frac{1}{\gamma}\frac{\partial q_{kh}}{\partial q_{ij}} - \frac{1}{\gamma^2}q_{kh}\frac{\partial \gamma}{\partial q_{ij}}, \quad i,j,k,h \in \sS, 
\end{align*}

where $\gamma = \max_{k \in \sS} (q_{kk}) + \epsilon$, $\epsilon > 0$, and $q_{kk} = \sum_{i \in \sS} q_{ki}$, $k \in \sS$.
Note that if $q_{ij}$ is not one of the rates in $\gamma$, then
\[
	\frac{\partial \mP(\mQ)}{\partial q_{ij}} = \left[ \begin{array}{ccccc}
        {\bf 0} & {\bf 0} & \cdots &  {\bf 0} & {\bf 0} \\
        \cdots & -1/\gamma & \cdots & 1/\gamma & \cdots \\
        {\bf 0} & {\bf 0} & \cdots &  {\bf 0} & {\bf 0}
    \end{array}
    \right],
\]
where $\frac{\partial \mP(\mQ)}{\partial q_{ij}}$ is a matrix of zeroes except at elements $\left(\frac{\partial \mP(\mQ)}{\partial q_{ij}} \right)_{ii} = - 1/\gamma$ and $\left(\frac{\partial \mP(\mQ)}{\partial q_{ij}} \right)_{ij} = 1/\gamma$.

If $q_{ij}$ is one of the rates in $\gamma$, then
\[
	\frac{\partial \mP(\mQ)}{\partial q_{ij}} = \frac{1}{\gamma^2} \left[ \begin{array}{cccccccc}
        q_{11}   & -q_{12}  & \cdots & -q_{1i}  & \cdots & -q_{1j}    & \cdots & -q_{1|\sS|}  \\
        \vdots & \vdots & \cdots & \vdots & & \vdots & \cdots &\vdots \\
        - q_{i1}  & -q_{i2}  & \cdots & \gamma + q_{ii}  & \cdots & -\gamma - q_{ij}   & \cdots & - q_{i|\sS|}  \\
        \vdots & \vdots & \cdots & \vdots & & \vdots & \cdots & \vdots \\
        -q_{|\sS|1}    & -q_{|\sS|2}  & \cdots & -q_{|\sS|i}   & \cdots & -q_{|\sS|j}    & \cdots & q_{|\sS||\sS|}  
    \end{array}
    \right],
\]
noting that $\frac{\partial \mP(\mQ)}{\partial q_{ij}} {\bf 1} = {\bf 0}$ regardless of whether $q_{ij}$ is in $\gamma$ or not.

To simplify the notation, denote $\mP \equiv \mP(\mQ)$. Now note that 
\begin{equation*}
\begin{aligned}
	\lim\limits_{t \rightarrow \infty}{\sum\limits_{l = 1}^{t}{\left( \mP^{t - l}\frac{\partial \mP}{\partial q_{ij}}\mP^{l - 1} \right)}}
	&= \lim\limits_{t' \rightarrow \infty}{\sum\limits_{l = 1}^{2t'}{\left( \mP^{2t' - l}\frac{\partial \mP}{\partial q_{ij}}\mP^{l - 1} \right)}} \\
	&= \lim\limits_{t' \rightarrow \infty}{\sum\limits_{l = 1}^{t'}{\left( \mP^{2t' - l}\frac{\partial \mP}{\partial q_{ij}}\mP^{l - 1} + \mP^{2t' - (2t' + 1 - l)}\frac{\partial \mP}{\partial q_{ij}}\mP^{(2t' + 1 - l) - 1} \right)}} \\
	&= \lim\limits_{t' \rightarrow \infty}{\sum\limits_{l = 1}^{t'}{\left( \mP^{2t' - l}\frac{\partial \mP}{\partial q_{ij}}\mP^{l - 1} + \mP^{l - 1}\frac{\partial \mP}{\partial q_{ij}}\mP^{2t' - l} \right)}} \\
	&= \lim\limits_{t' \rightarrow \infty}{\sum\limits_{l_1 = 1}^{t'}{\left( \mP^{2t' - l_1}\frac{\partial \mP}{\partial q_{ij}}\mP^{l_1 - 1} \right)}} + \lim\limits_{t' \rightarrow \infty}{\sum\limits_{l_2 = 1}^{t'}{\left( \mP^{l_2 - 1}\frac{\partial \mP}{\partial q_{ij}}\mP^{2t' - l_2} \right)}} \\
	&= \sum\limits_{l_1 = 1}^{\infty}{\left( \boldsymbol{\Pi}\frac{\partial \mP}{\partial q_{ij}}\mP^{l_1 - 1} \right)} + \sum\limits_{l_2 = 1}^{\infty}{\left( \mP^{l_2 - 1}\frac{\partial \mP}{\partial q_{ij}}\boldsymbol{\Pi} \right)} ,
\end{aligned}
\end{equation*}

where $\boldsymbol{\Pi} \equiv \lim_{t' \to \infty} \mP^{t'}$.
The r.h.s.\ of Equation~\eqref{eq:infinitySplit} is zero because $\frac{\partial \mP(\mQ)}{\partial q_{ij}} {\bf 1} = {\bf 0}$, where ${\bf 1}$ is a column vector of ones, and 
\begin{equation} \label{eq:matrix-pi}
	\boldsymbol{\Pi} \equiv \lim_{t' \to \infty} \mP^{t'} = \left[ \begin{array}{c} \bpi^\mathsf{T} \\ \vdots \\ \bpi^\mathsf{T} \end{array} \right].
\end{equation}
Noting that
\begin{equation} \label{eq:converge-key}
	\lim_{l \to \infty} \frac{\partial \mP}{\partial q_{ij}}\mP^{l - 1} = \frac{\partial \mP}{\partial q_{ij}} \boldsymbol{\Pi} = \frac{\partial \mP(\mQ)}{\partial q_{ij}} {\bf 1} \bpi^\mathsf{T} = 0.
\end{equation} 
Thus,
\begin{equation*}
\begin{aligned}
	\lim\limits_{t \rightarrow \infty}{\sum\limits_{l = 1}^{t}{\left( \mP^{t - l}\frac{\partial \mP}{\partial q_{ij}}\mP^{l - 1} \right)}}
	&= \sum\limits_{l_1 = 1}^{\infty}{\left( \boldsymbol{\Pi}\frac{\partial \mP}{\partial q_{ij}}\mP^{l_1 - 1} \right)} + \sum\limits_{l_2 = 1}^{\infty}{\left( \mP^{l_2 - 1}\frac{\partial \mP}{\partial q_{ij}}\boldsymbol{\Pi} \right)} \\
	&= \sum\limits_{l_1 = 1}^{\infty}{\left( \boldsymbol{\Pi}\frac{\partial \mP}{\partial q_{ij}}\mP^{l_1 - 1} \right)} \\
	&= \boldsymbol{\Pi}\sum\limits_{l_1 = 1}^{\infty}{\left( \frac{\partial \mP}{\partial q_{ij}}\mP^{l_1 - 1} \right)}
\end{aligned}
\end{equation*}
will converge by two reasons: summation controlled by $l_2$ equals 0 according to \Eqref{eq:converge-key}; and term inside summation controlled by $l_1$ converges to 0 as $l_1 \to 0$ according to \Eqref{eq:converge-key}.

\end{proof}

%
%
%
%
\subsection*{\AppProC: Proof of Theorem~\ref{thm:infSGD}} \label{subsec:app:proc}

We restate the theorem for the sake of completeness.
\begin{theorem*}[Infinity Stochastic Gradient Descent ($\infty$-SGD)]
Let $\mQ$ be the transition rate matrix of a stationary and ergodic CTMC. Assume strictly positive values for the learnable parameters $\vtheta^{(h)}$ at the $h$-th step of the optimization. Let $\mP(\mQ(\vx, \vtheta^{(h)}))$ be its uniformized transition probability matrix per Definition~\ref{def:unif}.
Let $\Loss(\vy,\bpi)$ be as in Equation~\eqref{eq:loss}.
%
Reparameterize $\tilde{\Loss}(\vy, \vx, \vtheta^{(h)}) = \Loss(\vy,\bpi(\vx,\vtheta^{(h)}))$ as the loss function with respect to $\vtheta^{(h)}$.
Let $X^{(h)} \sim \text{Geometric}(p^{(h)})$, $X^{(h)} \in \sZ^+$, be an independent sample of  a Geometric distribution with $p^{(h)} < \delta^{(h)}$, where $\delta^{(h)}$ is the spectral gap of $\mP(\mQ(\vx, \vtheta^{(h)}))$.
Then, for $0 < \epsilon \ll 1$ and for all learnable parameters $\vtheta$,
\begin{align*}
	\vtheta^{(h+1)}_{k} &= \max\left(\vtheta^{(h)}_k - \eta^{(h)} \nabla_{\vtheta_k} \left. \tilde{\Loss}(\vy,\vx,\vtheta) \right\vert_{\vtheta=\vtheta^{(h)}}, \epsilon\right), 
\end{align*}
where
	%
\begin{align*} 
	\nabla_{\vtheta_k} \tilde{\Loss}(\vy,\vx,\vtheta) &= \sum\limits_{ij} \sum\limits_{mn} (\rvp^\text{(events)}(0))_m \left. \frac{\partial \Loss(\vy,\bpi)}{\partial \bpi_{n}} \right\vert_{\bpi=\bpi(\vx,\vtheta)}\\
	&\qquad \times \bpi(\vx,\vtheta)_{n} \Gamma_{ijmn}(\vx, \vtheta) \frac{\partial \mQ(\vx,\vtheta)_{ij}}{\partial \vtheta_k}
\end{align*}
with $h = 0,1,\ldots,$ where $\bpi(\vx,\vtheta)$ is the steady state distribution defined in Equation~\eqref{eq:pi}, $\eta^{(h)}$ is the learning rate with $\sum_{h=0}^\infty \eta^{(h)} = \infty$, $\sum_{h=0}^\infty \left(\eta^{(h)}\right)^2 < \infty$, and
\begin{equation*}
   \!\! \Gamma_{ijmn}(\vx,\vtheta)\! =\! \sum_{t = 0}^{X^{(h)}}\! \left[ \frac{\partial \mP(\mQ(\vx, \vtheta))}{\partial q_{ij}}{\frac{\mP(\mQ(\vx, \vtheta))^t}{\Prob[X^{(h)} > t]}} \right]_{mn}\!\!,
\end{equation*}
is a stochastic gradient descent method that minimizes Equation~\eqref{eq:MLE}.
\end{theorem*}

\begin{proof}
We first derive the partial gradients $\nabla_{\vtheta_k} \tilde{\Loss}(\vy,\vx,\vtheta)$ from bottom. Similar to \Eqref{eq:matrix-pi}, we define $\boldsymbol{\Pi} = \Puni(\mQ(\vx,\vtheta))^{\infty}$ which is also a stack of steady state distribution $\pi(\vx,\vtheta)^\mathsf{T}$.
\resizebox{\linewidth}{!}{
\begin{minipage}{\linewidth}
\begin{equation*}
\begin{aligned}
    \nabla_{\vtheta_k} \tilde{\Loss}(\vy,\vx,\vtheta)
    &= \nabla_{\vtheta_k} \Loss(\vy,\bpi(\vx,\vtheta^{(h)})) \\
    &= \nabla_{\vtheta_k} \Loss(\vy,\rvp^\text{(events)}(0)^\mathsf{T} \boldsymbol{\Pi}(\vx,\vtheta)) \\
    &= \sum\limits_{i,j \in \sS} \frac{\partial \Loss(\vy,\rvp^\text{(events)}(0)^\mathsf{T} \Puni(\mQ(\vx,\vtheta))^{\infty})}{\partial q_{ij}} \frac{\partial \mQ(\vx,\vtheta)_{ij}}{\partial \vtheta_k} \\
    &= \sum\limits_{i,j \in \sS} \left[ \sum\limits_{m,n \in \sS} \frac{\partial \Loss(\vy,\rvp^\text{(events)}(0)^\mathsf{T} \boldsymbol{\Pi}(\vx,\vtheta))}{\partial \boldsymbol{\Pi}_{mn}} \frac{\partial \boldsymbol{\Pi}(\vx,\vtheta)_{mn}}{\partial q_{ij}} \right] \frac{\partial \mQ(\vx,\vtheta)_{ij}}{\partial \vtheta_k} \\
    &= \sum\limits_{i,j \in \sS} \sum\limits_{m,n \in \sS} (\rvp^\text{(events)}(0)^\mathsf{T})_m \frac{\partial \Loss(\vy, \bpi(\vx,\vtheta))}{\partial \bpi_n} \frac{\partial \boldsymbol{\Pi}(\vx,\vtheta)_{mn}}{\partial q_{ij}} \frac{\partial \mQ(\vx,\vtheta)_{ij}}{\partial \vtheta_k} \\
    &= \sum\limits_{i,j \in \sS} \sum\limits_{m,n \in \sS} (\rvp^\text{(events)}(0)^\mathsf{T})_m \frac{\partial \Loss(\vy, \bpi(\vx,\vtheta))}{\partial \bpi_n} \left[ \frac{\partial \boldsymbol{\Pi}(\vx,\vtheta)}{\partial q_{ij}} \right]_{mn} \frac{\partial \mQ(\vx,\vtheta)_{ij}}{\partial \vtheta_k} \\
    &= \sum\limits_{i,j \in \sS} \sum\limits_{m,n \in \sS} (\rvp^\text{(events)}(0)^\mathsf{T})_m \frac{\partial \Loss(\vy, \bpi(\vx,\vtheta))}{\partial \bpi_n} \left[ \left( \frac{\partial \Puni(\mQ(\vx,\vtheta))}{\partial \mQ} \right)_{ij} \right]_{mn} \frac{\partial \mQ(\vx,\vtheta)_{ij}}{\partial \vtheta_k} \\
    &= \sum\limits_{i,j \in \sS} \sum\limits_{m,n \in \sS} (\rvp^\text{(events)}(0)^\mathsf{T})_m \frac{\partial \Loss(\vy, \bpi(\vx,\vtheta))}{\partial \bpi_n} \left[ \boldsymbol{\Pi}(\vx,\vtheta) \sum_{l = 0}^{\infty} \frac{\partial \Puni(\mQ(\vx,\vtheta))}{\partial q_{ij}} \Puni(\mQ(\vx,\vtheta))^{l} \right]_{mn} \frac{\partial \mQ(\vx,\vtheta)_{ij}}{\partial \vtheta_k} \\
\end{aligned}
\end{equation*}
\end{minipage}
}
%
In the proof we will use $X$ rather than $X^{(h)}$, whose meaning will be clear from context. We then show that the above derivative proposed in the theorem is an unbiased estimation of above derivative.
%
%

First, assume we imposed the condition (we will later prove this to be true)
\begin{equation} \label{eq:cond}
 \sum_{x=1}^\infty \frac{\left\Vert \sum_{j = x+1}^{\infty} \frac{\partial \mP(\mQ)}{\partial q_{ij}}{ \mP^j} \right\Vert_2^2}{\Prob[X \geq x]}  < \infty, \quad \forall i,j,
 \end{equation}
we imposed is equivalent to the condition 
\begin{equation*}
 \sum_{x=1}^\infty \frac{\left\Vert \sum_{j = 1}^{\infty} \frac{\partial \mP(\mQ)}{\partial q_{ij}}{ \mP^j} - \sum_{j = 1}^{x} \frac{\partial \mP(\mQ)}{\partial q_{ij}}{ \mP^j} \right\Vert_2^2}{\Prob[X \geq x]} < \infty,
\end{equation*}
which we use to invoke Theorem 1 of~\citet{rhee2015unbiased}, which shows that under these conditions the expectation exists, \ie,
$\boldsymbol{\Pi} \: \expected_X[\Gamma_{ij}]=(\nabla^{(\infty)}_\mQ \mP(\mQ))_{ij}$, with $\nabla^{(\infty)}_\mQ \mP(\mQ)$ as defined in Equation~\eqref{eq:infinitySplit}, and $\Gamma_{ij}$ has also a finite second moment.
Since the expectation exists, the second part of the proof starts with the expansion of the expectation
\[
\expected_X [\Gamma_{ij}] =  \expected_X \left[\sum_{l' = 0}^{X} \frac{\frac{\partial \mP}{\partial q_{ij}}\mP^{l'}}{\Prob[X > l']} \right] =  \sum_{x = 1}^{\infty}\sum_{l = 1}^{x}   \frac{\frac{\partial \mP}{\partial q_{ij}} \mP^{l - 1}} {\Prob[X \geq l]} \Prob[X = x].
\]
By Fubini’s theorem 
\begin{align*}
 &   \sum_{x = 1}^{\infty}\sum_{l = 1}^{x}   \frac{\frac{\partial \mP}{\partial q_{ij}} \mP^{l - 1}} {\Prob[X \geq l]} \Prob[X = x] 
=   \sum\limits_{l = 1}^{\infty} \sum\limits_{x = l}^{\infty}   \frac{\frac{\partial \mP}{\partial q_{ij}} \mP^{l - 1}} {\Prob[X \geq l]}\Prob[X = x], 
\end{align*}
while moving the sum that depends on $x$ inside, noting that $\sum_{x=l}^{\infty}{\Prob[X=x]} = \Prob[X \geq x]$, and canceling the terms yields
\begin{align*}
& \sum_{l = 1}^{\infty} \left( \frac{\frac{\partial \mP}{\partial q_{ij}} \mP^{l - 1}} {\Prob[X \geq l]} \sum_{x = l}^{\infty}\Prob[X = x] \right) \\
&=   \sum_{l = 1}^{\infty} \left( \frac{\frac{\partial \mP}{\partial q_{ij}} \mP^{l - 1}}{\Prob[X \geq l]} \Prob[X \geq l] \right) \\
&=   \sum_{l' = 0}^{\infty} \frac{\partial \mP}{\partial q_{ij}} \mP^{l'}. 
\end{align*}
This part of the proof concludes by noting that
$$
\bpi(\vx,\vtheta)_{n} \sum_{t = 0}^{X^{(h)}} \left[ \frac{\partial \mP(\mQ(\vx, \vtheta))}{\partial q_{ij}}{\frac{\mP(\mQ(\vx, \vtheta))^t}{\Prob[X^{(h)} > t]}} \right]_{mn}
$$
is an unbiased estimator of the gradient
$$
\left[ \boldsymbol{\Pi}(\vx,\vtheta) \sum_{l = 0}^{\infty} \frac{\partial \Puni(\mQ(\vx,\vtheta))}{\partial q_{ij}} \Puni(\mQ(\vx,\vtheta))^{l} \right]_{mn}
$$
with finite variance.
As also the derivative of the $\log$ in ${\partial \Loss(\vy, \bpi(\vx,\vtheta))} / {\partial \bpi_n}$ , and the derivative ${\partial \mQ(\vx,\vtheta)_{ij}} / {\partial \vtheta_k}$ are infinitely  differentiable.
Thus, the gradient estimate $\nabla_{\vtheta_k} \tilde{\Loss}(\vy,\vx,\vtheta)$ can be used to find a fixed point in a Robbins-Monro stochastic optimization procedure~\citep{Bottou,Bottou1998} of Equation~\eqref{eq:MLE}.

We now prove that Equation~\eqref{eq:cond} is satisfied by our choice of $p^{(h)} < \delta^{(h)}$, where $\delta^{(h)}$ is the spectral gap of $\mP$.
In what follows we drop the dependence on $(h)$ to simplify the notation.
Note that we need to show
\[
\sum_{x=1}^\infty \frac{\left\Vert \sum_{j = x+1}^{\infty} \frac{\partial \mP(\mQ)}{\partial q_{ij}} \mP^j \right\Vert_2^2}{\Prob[X \geq x]}  < \infty , \quad \text{ for all learnable $q_{ij}$}.
\]
Without loss of generality, assume a fixed $i$ and $j$. Because $\mP \equiv \mP(\mQ_t)$ is stationary and ergodic for any positive values learnable parameters of $\mQ_t$, the spectral decomposition of $\mP = V \Lambda V^{-1}$ has eigenvalues $1 > \lambda_2 \geq \ldots \geq \lambda_{|\sS|} > -1$. And let $\delta = 1 - \max(|\lambda_2|,\ldots,|\lambda_{|\sS|}|)$ be the spectral gap of $\mP$.
We start the proof using the spectral decomposition in the series:
 \begin{align*}
 \sum_{x=1}^\infty \frac{\left\Vert \sum_{j = x+1}^{\infty} \frac{\partial \mP(\mQ)}{\partial q_{ij}} \mP^j \right\Vert_2^2}{\Prob[X \geq x]}  
& =  \sum_{x=1}^\infty \frac{\left\Vert \sum_{j = x+1}^{\infty} \frac{\partial \mP(\mQ)}{\partial q_{ij}} V \Lambda^j V^{-1} \right\Vert_2^2}{\Prob[X \geq x]}  \\
&  =  \sum_{x=1}^\infty \frac{\left\Vert \sum_{j = x+1}^{\infty} \lambda_i^j \langle \frac{\partial \mP(\mQ)}{\partial q_{ij}}, V_{\cdot i}\rangle (V^{-1})_{\cdot i} \right\Vert_2^2}{\Prob[X \geq x]}  \qquad (a),
\end{align*}
because $ V_{\cdot 1} = \bpi$  and $\langle \frac{\partial \mP(\mQ)}{\partial q_{ij}} ,\bpi \rangle = 0$ (see proof of Proposition~\ref{prop:infgrad}), then $\langle \frac{\partial \mP(\mQ)}{\partial q_{ij}} ,V_{\cdot 1} \rangle = 0$, which yields
\begin{align*}
   (a)  =  \sum_{x=1}^\infty \frac{\left\Vert \sum_{j = x+1}^{\infty} \sum_{i=2}^{|\sS|} \lambda_i^j \langle \frac{\partial \mP(\mQ)}{\partial q_{ij}}, V_{\cdot i}\rangle (V^{-1})_{\cdot i} \right\Vert_2^2}{\Prob[X \geq x]}  \qquad (b)
\end{align*}
Let $X \sim \text{Geometric}(p)$, $p \in (0,1)$ with average $E[X]=1/p$.
Then, $\exists C(\mQ)$ s.t.\
\begin{align*}
(b)& =   \sum_{x=1}^\infty \frac{\left\Vert \sum_{j = x+1}^{\infty} \sum_{i=2}^{|\sS|} \lambda_i^j \langle \frac{\partial \mP(\mQ)}{\partial q_{ij}}, V_{\cdot i}\rangle (V^{-1})_{\cdot i} \right\Vert_2^2}{(1-p)^x} \\
&\leq \sum_{x=1}^\infty \frac{\left\Vert \sum_{j = x+1}^{\infty} \sum_{i=2}^{|\sS|} |\lambda_i|^j C(\mQ) \right\Vert_2^2}{(1-p)^x}
\end{align*}
and because $(1- \delta) >  |\lambda_i|$, for $i \geq 2$, we have 
\begin{align*}
&\sum_{x=1}^\infty \frac{\left\Vert \sum_{j = x+1}^{\infty} \sum_{i=2}^{|\sS|} |\lambda_i|^j C(\mQ) \right\Vert_2^2}{(1-p)^x}\\
&\leq \sum_{x=1}^\infty \frac{\left\Vert \sum_{j' = 1}^{\infty} (1-\delta)^{j'+x} \: |\sS| \: C(\mQ) \right\Vert_2^2}{(1-p)^x}\\
&=  \sum_{x=1}^\infty  \left(\frac{(1-\delta)^2}{1-p}\right)^x \left\Vert \sum_{j' = 1}^{\infty} (1-\delta)^{j'} \: |\sS| \: C(\mQ) \right\Vert_2^2
\end{align*}
From the Cauchy–Schwarz inequality,
\begin{align*}
    &\sum_{x=1}^\infty  \left(\frac{(1-\delta)^2}{1-p}\right)^x \left\Vert \sum_{j' = 1}^{\infty} (1-\delta)^{j'} \: |\sS| \: C(\mQ) \right\Vert_2^2\\
    &\leq \sum_{x=1}^\infty  \left(\frac{(1-\delta)^2}{1-p}\right)^x \sum_{j' = 1}^{\infty} (1-\delta)^{2j'} \: |\sS|^2 \: C(\mQ)^2,
\end{align*}
and using the Cauchy product
\begin{align*}
   & \sum_{x=1}^\infty  \left(\frac{(1-\delta)^2}{1-p}\right)^x \sum_{j' = 1}^{\infty} (1-\delta)^{2j'} \: |\sS|^2 \: C(\mQ)^2
    = |\sS|^2 \: C(\mQ)^2 \sum_{k=1}^\infty  \beta_k 
\end{align*}
where
\[
\beta_k = \sum_{l=0}^k  \left(\frac{(1-\delta)^2}{1-p}\right)^l  ((1-\delta)^{2})^{k-l} = \sum_{l=0}^k  \frac{(1-\delta)^{2k}}{(1-p)^l} = (1-\delta)^{2k} \frac{(1-p)^{-k} - (1-p)}{p} . 
\]
Finally, putting all the terms together
\[
\sum_{x=1}^\infty \frac{\left\Vert \sum_{j = x+1}^{\infty} \frac{\partial \mP(\mQ)}{\partial q_{ij}} \mP^j \right\Vert_2^2}{\Prob[X \geq x]}   \leq |\sS|^2 \: C(\mQ)^2 \sum_{k=1}^\infty  (1-\delta)^{2k} \frac{(1-p)^{-k} - (1-p)}{p},
\]
which is a convergent  geometric series if $p < \delta$ as $0 \leq 1-\delta < 1$, concluding our proof.
\end{proof}
%
%
%
%
%
%

%
%
%
%
\subsection*{\AppExpA: Experimental Testbed} \label{subsec:app:testbed}

The VoLTE testbed consists of (a) a Session Initiation Protocol (SIP)~\cite{siprfc} server implementing VoLTE functionality deployed using Kamailio (version version 5.0.4)~\citep{Kamailio}, and (b) a workload generator, SIPP~\citep{sipp}, which generates a predefined number of SIP \texttt{REGISTER} messages per second.
The VoLTE testbed 
is deployed on 2 HP ProLiant DL120G6 (Intel Xeon X3430 processor and 8 GB RAM) connected using one Gigabit Dell N2024 Switch. 
Both the SIP server and the workload generator run directly on a dedicated physical host.

\paragraph{System Architecture.}
The Kamailio server uses a pool of statically created threads to process requests from the client(s). Each thread reads data directly from the socket buffer, processes the SIP request and generates a response which is send to the client. Incoming request messages are stored in the Linux socket buffer before processing by the Kamailio application threads. Since the workload generator and SIP server communicate using User Datagram Protocol (UDP), any packets sent by the client after the socket buffer is full (Queue size $> K$) are discarded. The workload generator generates a predefined number of \texttt{REGISTER} messages each second ($\lambda$) and waits for a \texttt{200 OK} response from the SIP server. Any request that does not receive a successful \texttt{200 OK} response from the server within a predefined timeout (10~ms) is considered a failed call. 
In our experiments, the Kamailio server is configured to use a single application thread to process requests and the socket buffer is allocated to store up to 20 \texttt{REGISTER} requests (K=20). This setup therefore emulates an  M/M/1/$K$ queuing system with K=20.

\paragraph{Workload Characteristics}%
The workload generator generates traffic according to the request-rate distribution of a video streaming server of one of the largest video streaming operators in Brazil~\citep{Rocha2018CDN}. The request-rate distribution is shown in \Figref{fig:sipp_workoad}. We also measured the number of packets in the server queue during our experiments by instrumenting the workload generator code. An example of the number of packets enqueued at a server (k) is presented in  \Figref{fig:kam_k_size}. The figure presents the number of times an incoming request encountered a server queue size of 0 or 20 (maximum queue size in our testbed is 20). As evident from \Figref{fig:kam_k_size}, at lower workloads, the number of packets waiting in the queue when a new request arrives are frequently 0 (0-100 second), but at higher workloads, the server queue is frequently full (100-175 second).

\begin{figure} [htp]
\begin{subfigure}[t]{0.5\textwidth}
\centering
\includegraphics[width=\textwidth, keepaspectratio=true]{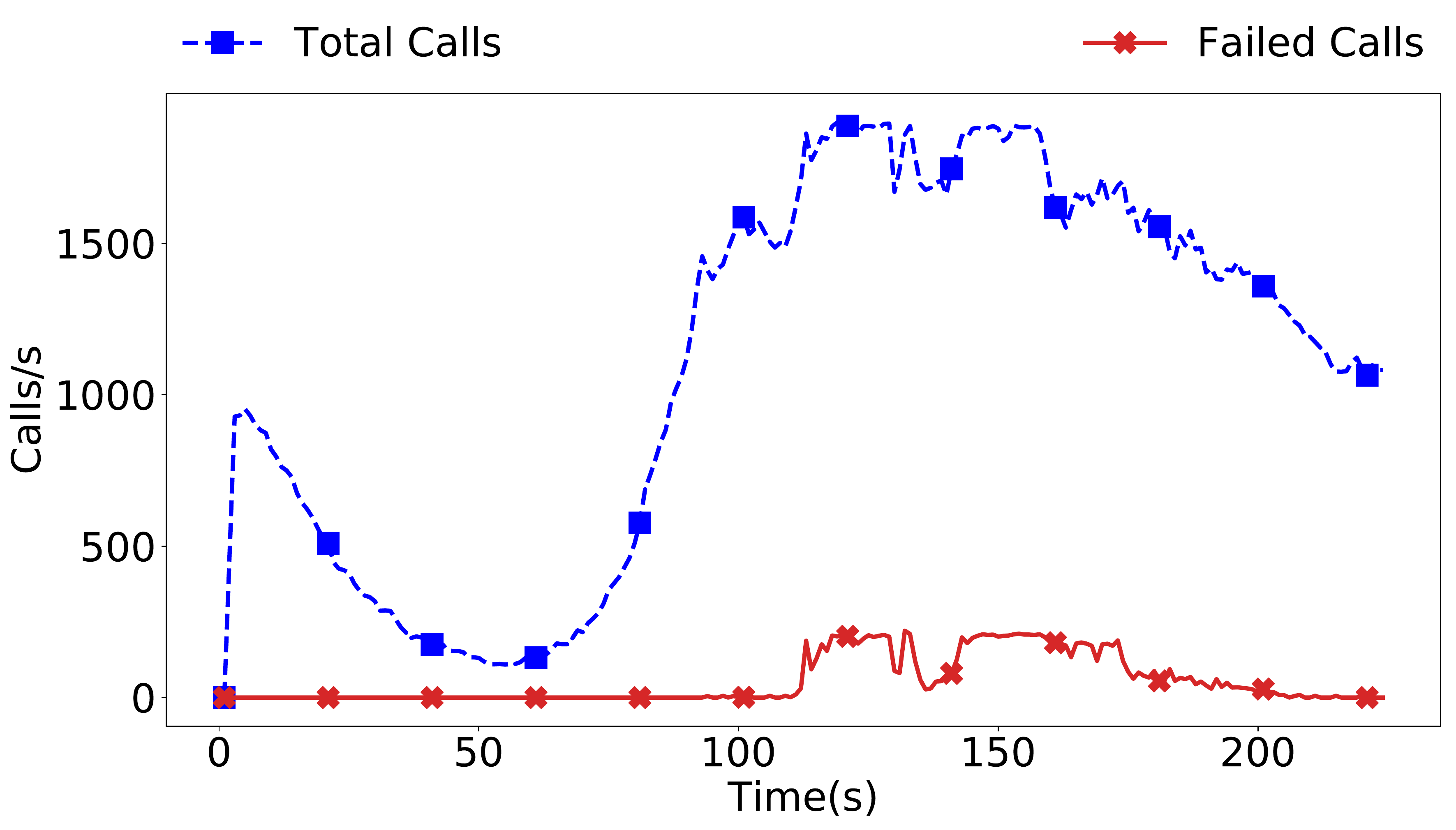}
\caption{Call distribution}
\label{fig:sipp_workoad}
\end{subfigure}
\hspace*{0.1in}
\begin{subfigure}[t]{0.5\textwidth}
\includegraphics[width=\textwidth, keepaspectratio=true]{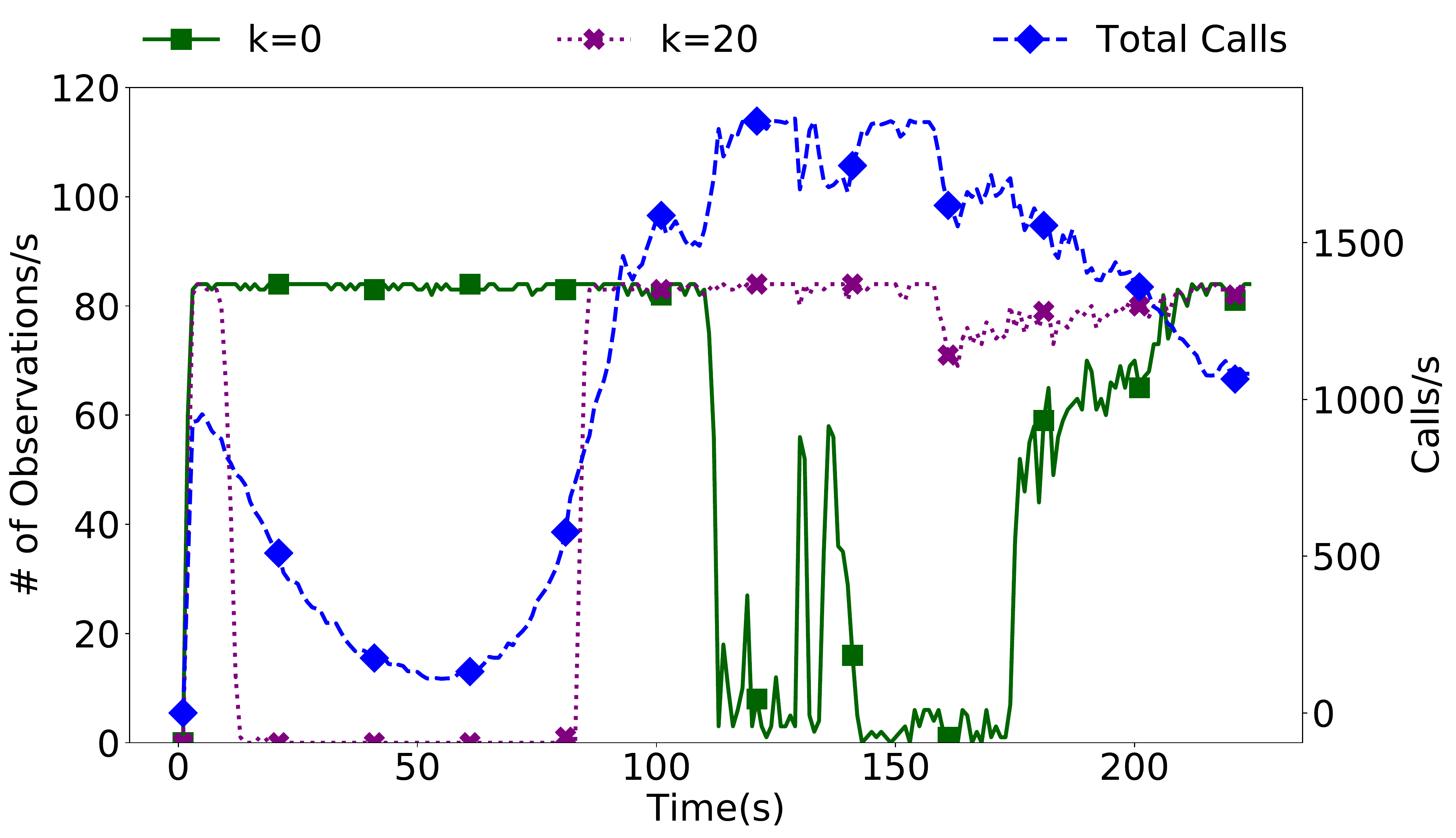}
\caption{Number of times a queue size of K is encountered by an incoming packet}
\label{fig:kam_k_size}
\end{subfigure}
\caption{Details from VoLTE testbed experiments}
\label{fig:kam_q_size}
\end{figure}

\paragraph{Data Collection Methodology.}%
In real production systems, it is not always possible to obtain the number of packets enqueued at the server. Applications/servers in a deployment environment typically do not generate statistics for system level socket buffers and internal queues. Therefore, for our experiments,  
we do not use the actual queue size generated by instrumenting the workload generator code. Instead, we estimate the number of packets in the server queue from the time taken by the server to respond to a request (server response time observed at the workload generator). 
Server response time is a combination of transmission delay, queueing delay, and processing time. That is, ResponseTime =   PropagationTime + QueueingTime + ProcessingTime. We use the time taken by the server to respond to a ping packet~\citep{Ping} as the PropagationTime, and the time taken by the server to process to the first request packet as the ProcessingTime. These quantities allow us to calculate the QueueingTime of a request from its ResponseTime. The server queue size observed by a request is then estimated from the QueueingTime as QueueSize =  QueueingTime/ProcessingTime.  
Assuming Poisson arrivals, the observed queue sizes are true samples from the time average (via PASTA property~\citep{wolff1982poisson}).

%
%
%
%
\subsection*{\AppExpB: Synthetic Training and Test Data} \label{subsec:app:simulate}

\paragraph{Queue Simulation Configurations.} At training time, we sample request rates of time window $n$ from uniform intervals that keep the system load light, while at test time, we sample from request rates of heavy loads.
For $n$-th training or test samples, we set $x^\text{train}_{n} \sim \text{Uniform}(\lambda^\text{train}_\text{min}, \lambda^\text{train}_\text{max})$, and test over $x^\text{test}_{n} \sim \text{Uniform}(\lambda^\text{test}_\text{min}, \lambda^\text{test}_\text{max})$.

In open queues (M/M/1/$K$ and M/M/m/m+r), at each time window of length $T$, the request rate that arrive at the system when there are zero or one packets in the queue is given by $X_0 \sim \text{Poisson}(\lambda T \bpi_0)$ and $X_1 \sim \text{Poisson}(\lambda T \bpi_1)$, respectively, where $\lambda$ is the request rate.
The latter sampling is a good approximation of the aggregate observations since $T$ is large enough for the system to reach steady state and, by the PASTA property~\citep{wolff1982poisson}, Poisson processes see time averages.
In the training data, we observe as few as 10 samples for each $\bpi$ distribution.

\paragraph{Queue Simulation Details.}
All the queue simulations share the same data generation process given in Algorithm~\ref{alg:gendata} where $N$ is number of samples to generate, $\lambda_\text{min}$ and $\lambda_\text{max}$ are inclusive boundaries for $x^\text{train}$ or $x^\text{test}$, and $\theta^\star$ is ground truth settings of all learnable parameters.

\begin{algorithm}
\caption{Data Generation Process}
\label{alg:gendata}
\begin{algorithmic} 
\REQUIRE{$N, \lambda_\text{min}, \lambda_\text{max}, \mQ(\cdot)$}
\STATE{// $\mQ(\cdot)$ is the queueing model that takes the request rate $x$ as input.}
\STATE{$T \leftarrow 1$}
\STATE{$n \leftarrow 0$}
\WHILE{$n < N$}
	\STATE{$x \sim \text{Uniform}(\lambda_\text{min}, \lambda_\text{max})$}
	\STATE{Get transition probability matrix $\mP$ and logging rate $\gamma$ of $\mQ(x; \theta^\star)$ as Definition~\ref{def:unif}}
	\STATE{Get steady-state distribution $\bpi$ of $\mQ(x; \theta^\star)$}
	\STATE{$t \leftarrow 0$}
	\STATE{$i \leftarrow 0$}
	\STATE{Sample initial state $s_{0}$ with respect to $\bpi$}
	\STATE{Sample interval between next event $d \sim \text{Exp}(\frac{1}{\gamma})$}
	\STATE{$t \leftarrow t + d$}
	\WHILE{$t < T$}
		\STATE{// Reach to next state.}
		\STATE{$i \leftarrow i + 1$}
		\STATE{Sample staying state of next event $s_{i}$ with respect to $P_{s_{i}:}$}
		\STATE{Sample interval between next event $d \sim \text{Exp}(\frac{1}{\gamma})$}
		\STATE{$t \leftarrow t + d$}
	\ENDWHILE
	\STATE{Count the appearance of $\{s_{0}, \cdots, s_{i}\}$ and append count to data}
	\STATE{$n \leftarrow n + 1$}
\ENDWHILE
\end{algorithmic}
\end{algorithm}

%
%
%
%
\subsection*{\AppExpC: Effect of Parametric Strength $\alpha$} \label{subsec:app:alpha}

In \Figref{fig:alpha-study}, we can see that strong parametric model is a clear winner among all experiments.
The failure of weak parametric models on extrapolation of all tasks may be caused by the failure on the convergence over training loss.
To exclude that factor, we further investigate the training loss on those tasks in Table~\ref{tab:alpha-loss}, and find that all parametric model strengths gives similar training loss.

This conclusion verifies the need of parametric model: without parametric model, our models can easily overfit on partly observed states $\sS'$, and lose the extrapolation ability on unseen states $\sS \backslash \sS'$.

\begin{figure}
\begin{minipage}{0.24\textwidth}
\centering
\includegraphics[width=\linewidth]{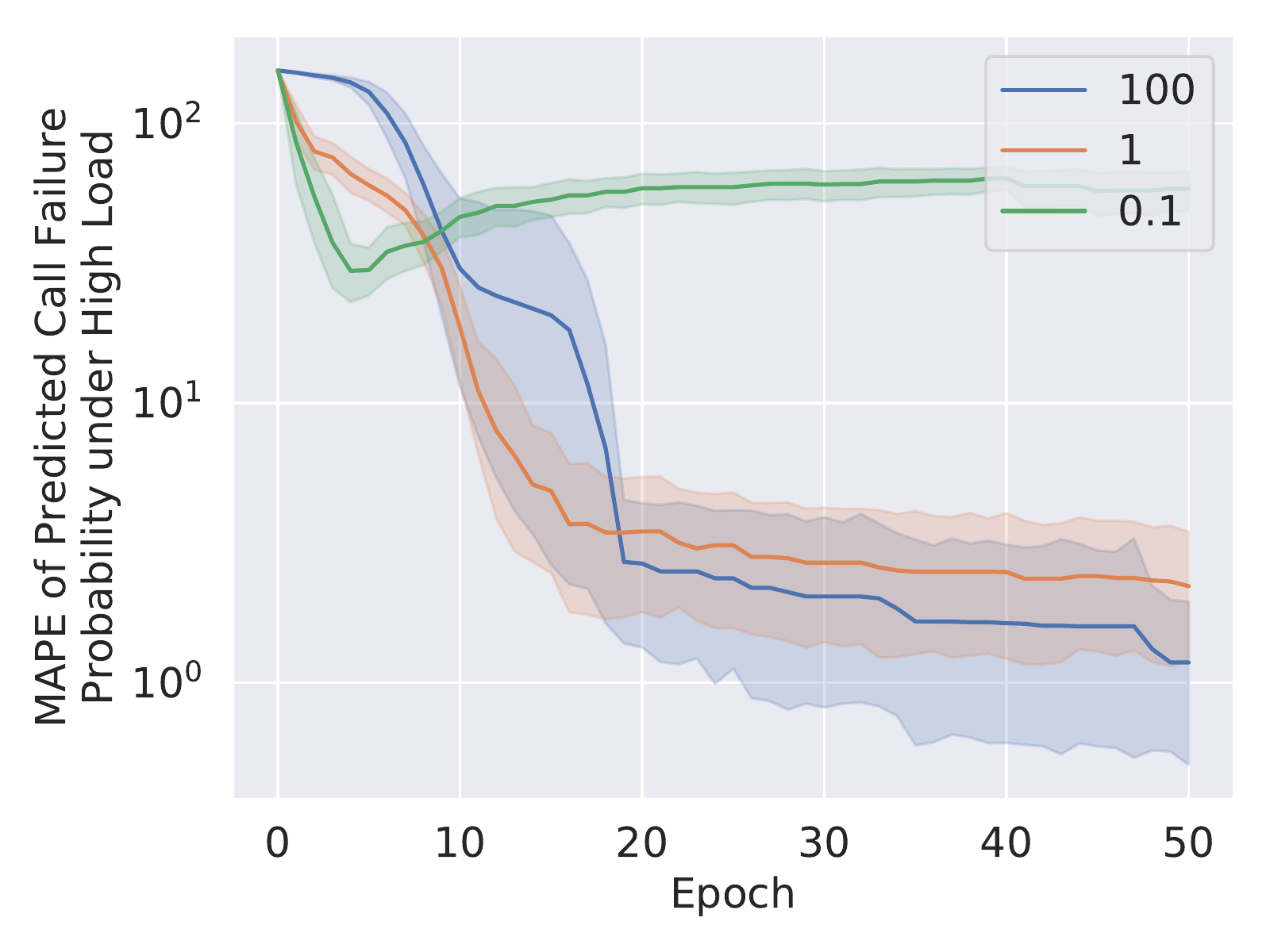}
\subcaption{Effect of parametric model on M/M/1/$K$ (fast mixing) ($\alpha \in \{0.1,1,100\}$).}
\end{minipage}
\hfill
\begin{minipage}{0.24\textwidth}
\centering
\includegraphics[width=\linewidth]{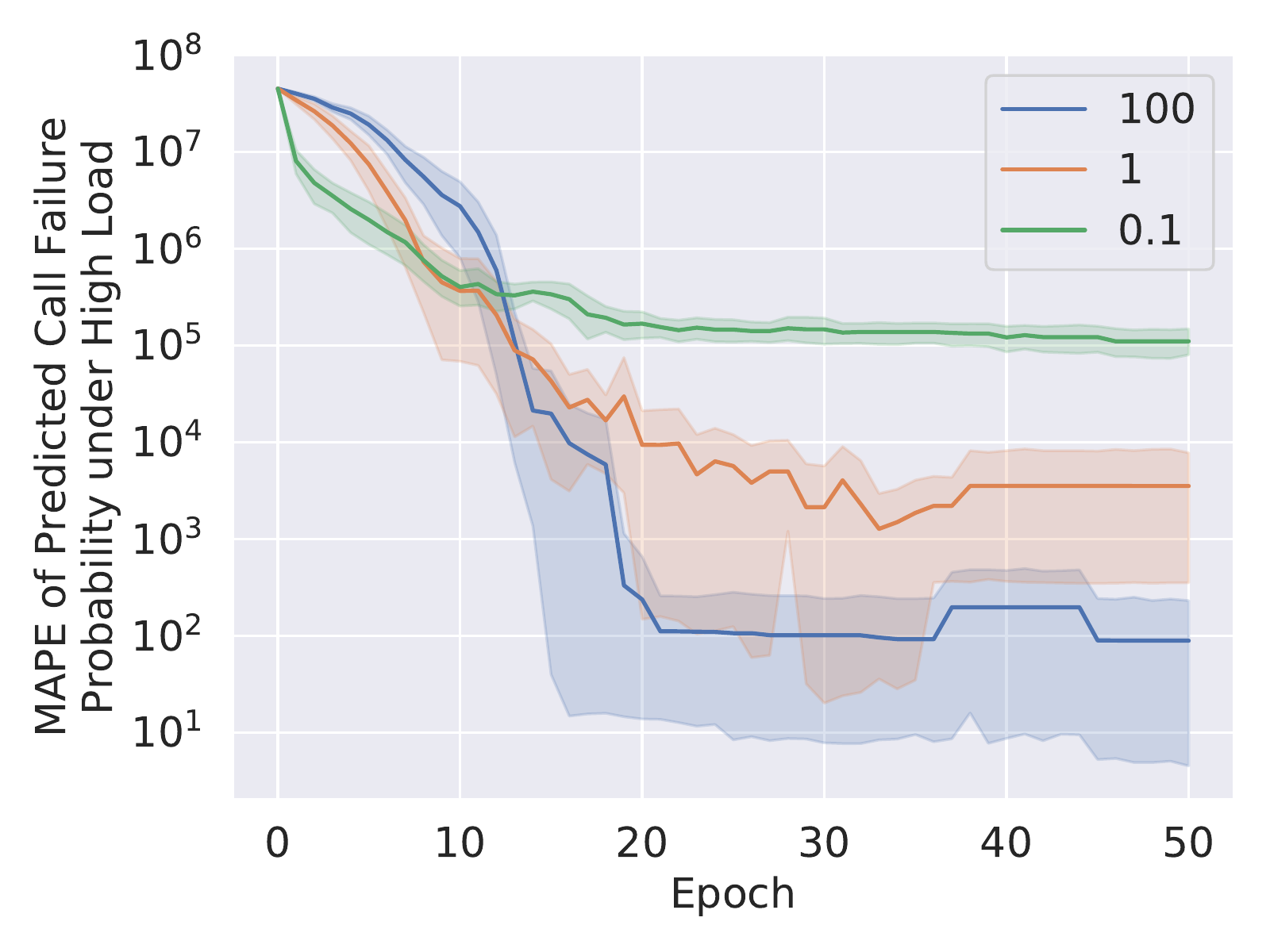}
\subcaption{Effect of parametric model on M/M/1/$K$ (slow mixing) ($\alpha \in \{0.1,1,100\}$).}
\end{minipage}
\hfill
\begin{minipage}{0.24\textwidth}
\centering
\includegraphics[width=\linewidth]{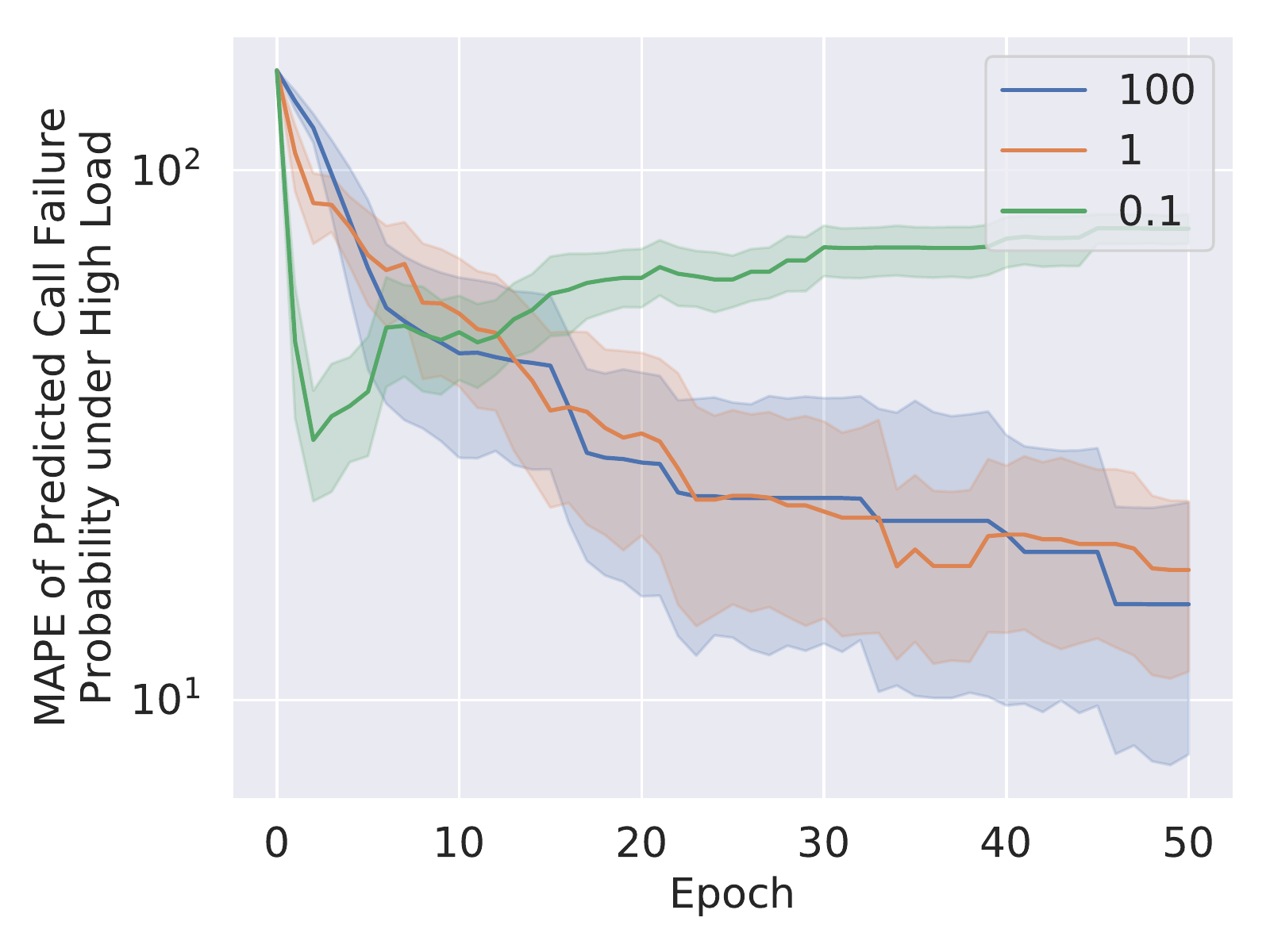}
\subcaption{Effect of parametric model on M/M/m/m+r ($\alpha \in \{0.1,1,100\}$).}
\end{minipage}
\hfill
\begin{minipage}{0.24\textwidth}
\centering
\includegraphics[width=\linewidth]{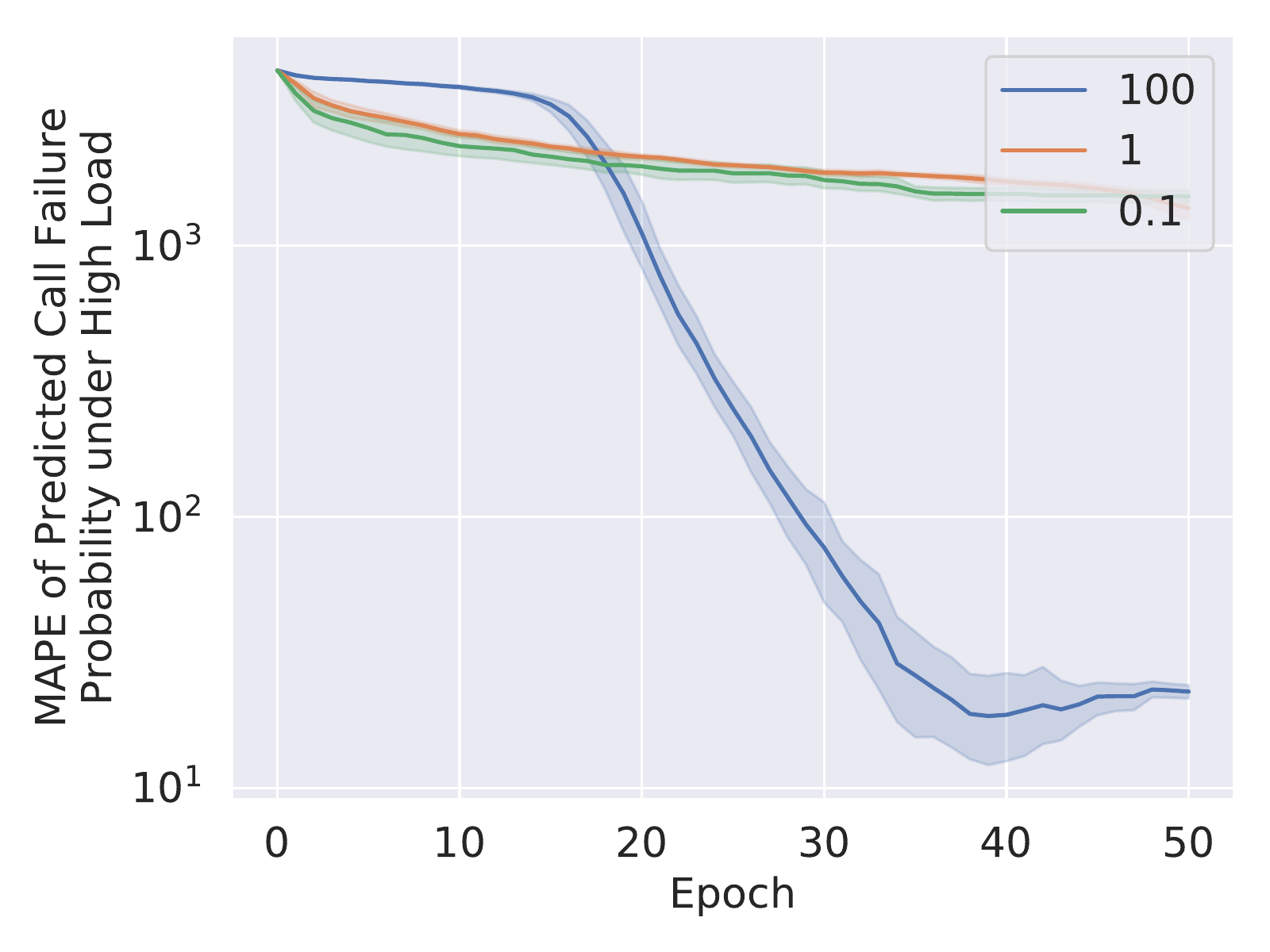}
\subcaption{Effect of parametric model on M/M/Multiple/$K$ ($\alpha \in \{0.1,1,100\}$).}
\label{fig:mmul-large-alpha}
\end{minipage}
\caption{Effect of parametric strength on different queues on test error (all optimized with $\infty$-SGD). (a) and (c) shows that strong strength ($\alpha=100$) sometimes may be same as weaker strength ($\alpha=1$). (b) and (d) shows that strong prior ($\alpha=100$) beats small prior ($\alpha=1$). (a-d) shows that too small strength ($\alpha=0.1$) will be harmful. (a-d) together shows that using strong prior ($\alpha=100$) does not harm, and may even achieve great improvement.}
\label{fig:alpha-study}
\end{figure}

\begin{table*}
\vspace{-5pt}
\caption{\small Simulation results showing training loss (negative log-likelihood). We can see that all training losses have similar mean and variance regardless of parametric strength $\alpha$. \label{tab:alpha-loss}}
\centering
\resizebox{1.\textwidth}{!}{
\begin{tabular}{@{\extracolsep{0pt}}lclll@{}}
& $\delta_n$ (spectral gap) & $\alpha = 0.1$ & $\alpha = 1$ & $\alpha = 100$ \\
\cmidrule(l){2-5}
M/M/1/$K$ (fast-mix) & [0.022, 0.043] & $ 3.59 \times 10^{1} \pm 7.33 \times 10^{-1} $ & $ 3.59 \times 10^{1} \pm 8.02 \times 10^{-1} $ & $ 3.60 \times 10^{1} \pm 7.45 \times 10^{-1} $ \\
M/M/1/$K$ (slow-mix) & [0.005, 0.008] & $ 8.11 ~~~~~~~~~~~ \pm 3.42 \times 10^{-1} $ & $ 8.12 ~~~~~~~~~~~ \pm 4.47 \times 10^{-1} $ & $ 8.22 ~~~~~~~~~~~ \pm 4.71 \times 10^{-1} $ \\
M/M/$m$/$m+r$        & [0.013, 0.024] & $ 2.22 \times 10^{1} \pm 1.64 ~~~~~~~~~~~ $ & $ 2.23 \times 10^{1} \pm 1.60 ~~~~~~~~~~~ $ & $ 2.19 \times 10^{1} \pm 1.39 ~~~~~~~~~~~ $ \\
M/M/Multiple/$K$     & [0.068, 0.096] & $ 9.30 \times 10^{1} \pm 1.46 ~~~~~~~~~~~ $ & $ 9.51 \times 10^{1} \pm 1.48 ~~~~~~~~~~~ $ & $ 9.33 \times 10^{1} \pm 1.60 ~~~~~~~~~~~ $ \\
\hline
\end{tabular}
}
\vspace{-10pt} 
\end{table*}

%
%
%
%
\subsection*{\AppExpD: Effect of Geometric Parameter $p^{(h)}$} \label{subsec:app:geop}

Theorem~\ref{thm:infSGD} gives a bound on $p^{(h)} < \delta^{(h)}$, where $\delta^{(h)}$ is the spectral gap of $\mP(\mQ(\vx,\vtheta^{(h)}))$ on the training data, to guarantee that $\infty$-SGD converges.
This bound is loose.
In our experiments, $\delta^{(h)}$ can always reach magnitude of 0.01, and even reach magnitude of 0.001 for slow mixing M/M/1/$K$.
This would require $p^{(h)} \leq 0.01$ to guarantee convergence, and for slow mixing M/M/1/$K$, it would be $p^{(h)} \leq 0.001$.
In the following experiments we see that such small values of $p^{(h)}$ are not required in practice.

Note that the smaller $p^{(h)}$ is, the more likely we will sample a large $X^{(h)}$, which will greatly increases the cost of computing $\Gamma_{ijmn}^{(h)}$ of \Eqref{eq:infGamma}.
We then test $\infty$-SGD with $p^{(h)}=0.1$ and $p^{(h)}=0.01$. For slow mixing M/M/1/$K$, $p^{(h)}=0.001$ is also included. The MAPE error results are provided in Figure~\ref{fig:geop-study}.
We can see that $p^{(h)}=0.01$ similar performance as $p^{(h)}=0.1$ on slow mixing M/M/1/$K$ and M/M/m/m+r, and fails to converge in 50 epochs on fast mixing M/M/1/$K$ and M/M/Multiple/$K$. In addition, for M/M/1/$K$, $p^{(h)}=0.001$ also gives similar result as $p^{(h)}=0.1$.
For the tasks where small $p^{(h)}$ converges, the time cost for each each epoch of $p^{(h)}=0.1$ is cut to half of $p^{(h)}=0.01$, and is cut to one fourth of $p^{(h)}=0.001$. For those where small $p^{(h)}$ does not converge, we may need even more epochs for $p^{(h)}=0.01$.
We note however, that the gains in training are not reflected in the generalization performance, as shown in the training losses for $p^{(h)}=0.1$ and $p^{(h)}=0.01$ in Table \ref{tab:geop-mse}.
Thus, we use $p^{(h)}=0.1$ in most of our experiments.

\begin{figure}
\begin{minipage}{0.24\textwidth}
\centering
\includegraphics[width=\linewidth]{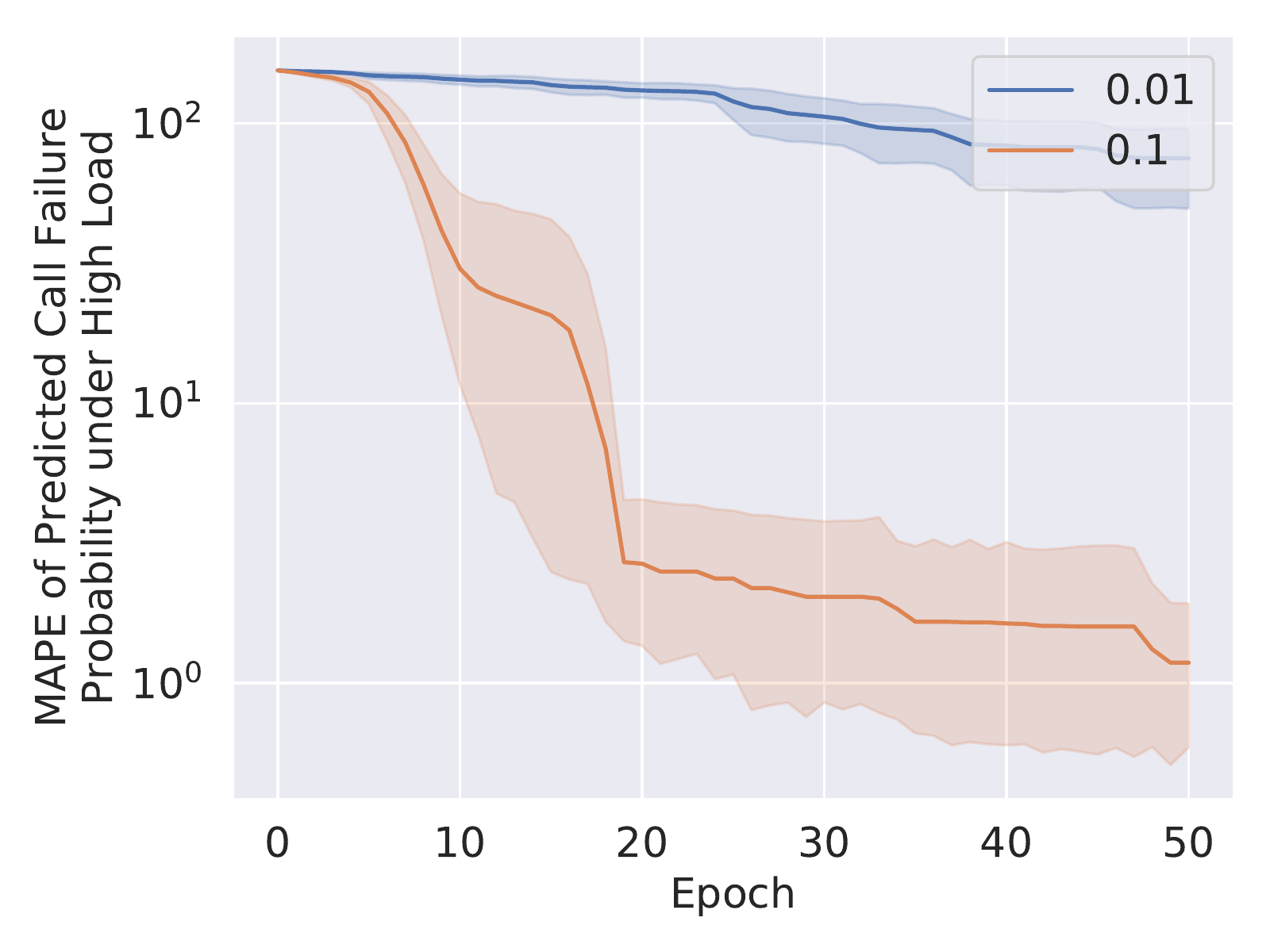}
\subcaption{Effect of geometric sampling on M/M/1/$K$ (fast-mix) ($p \in \{0.01,0.1\}$).}
\label{fig:mm1k-small-geop}
\end{minipage}
\hfill
\begin{minipage}{0.24\textwidth}
\centering
\includegraphics[width=\linewidth]{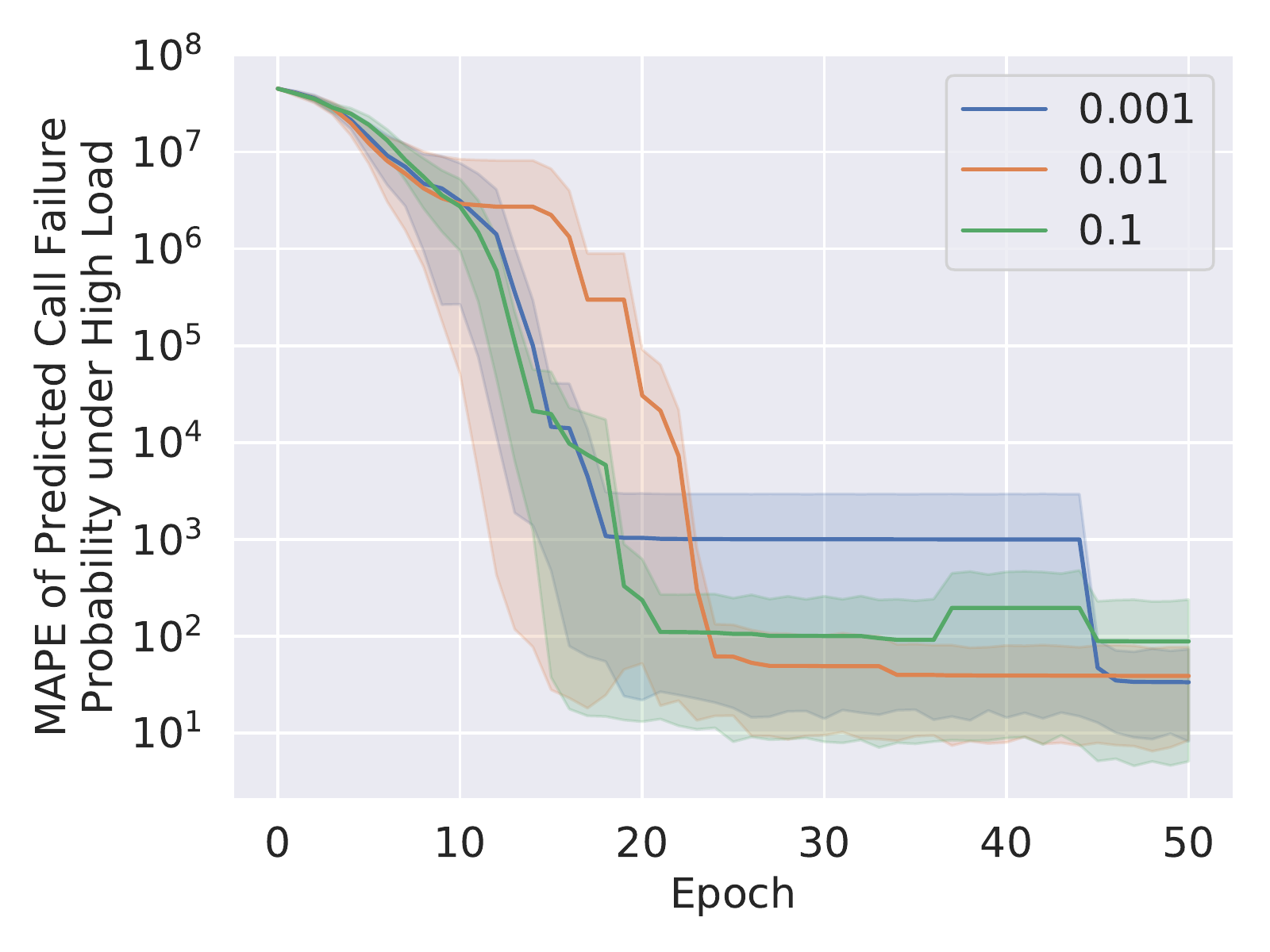}
\subcaption{Effect of geometric sampling on M/M/1/$K$ (slow-mix) ($p \in \{0.001,0.01,0.1\}$).}
\label{fig:mm1k-large-geop}
\end{minipage}
\hfill
\begin{minipage}{0.24\textwidth}
\centering
\includegraphics[width=\linewidth]{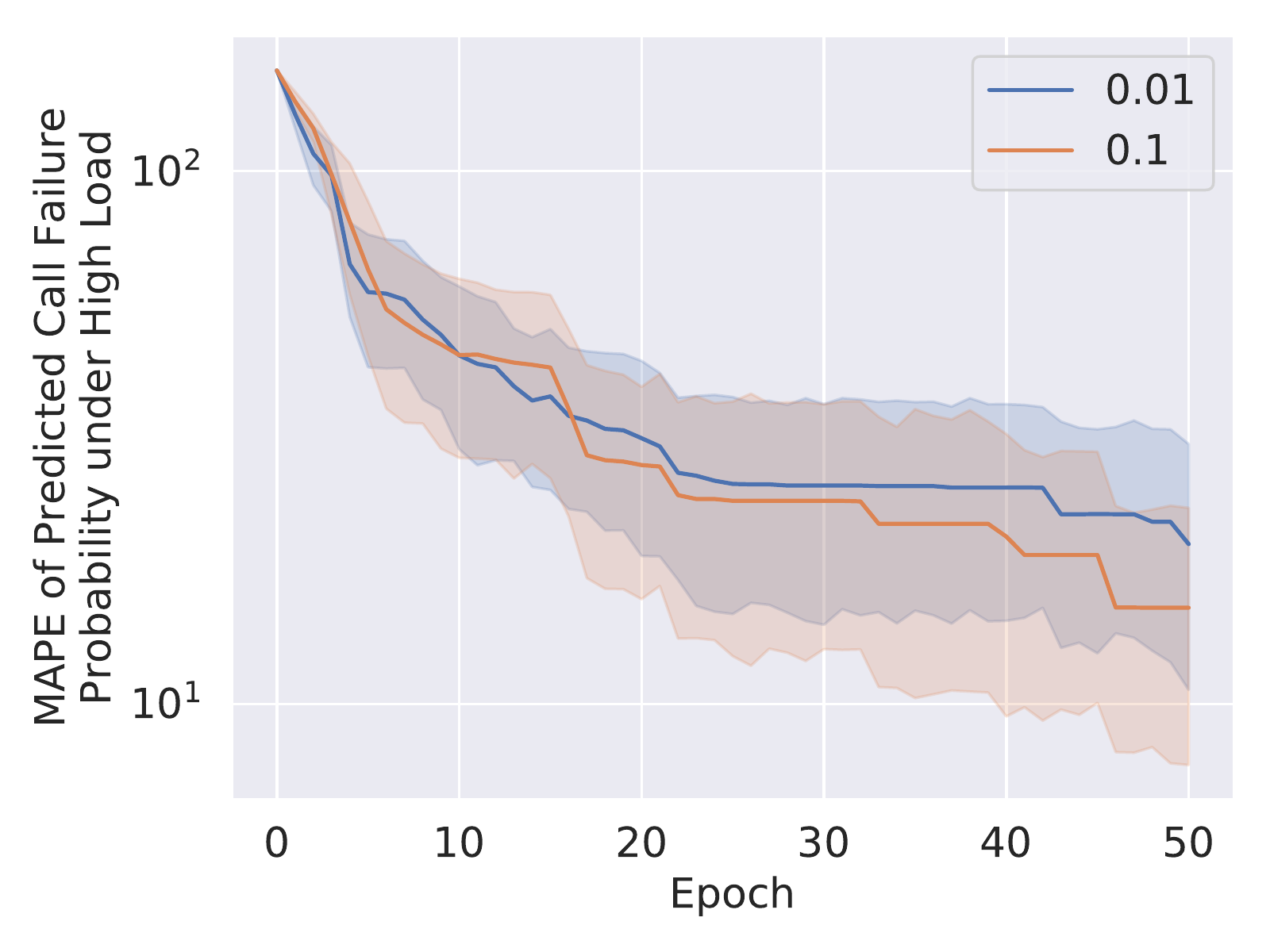}
\subcaption{Effect of geometric sampling on M/M/m/m+r ($p \in \{0.01,0.1\}$).}
\label{fig:mmmmr-geop}
\end{minipage}
\hfill
\begin{minipage}{0.24\textwidth}
\centering
\includegraphics[width=\linewidth]{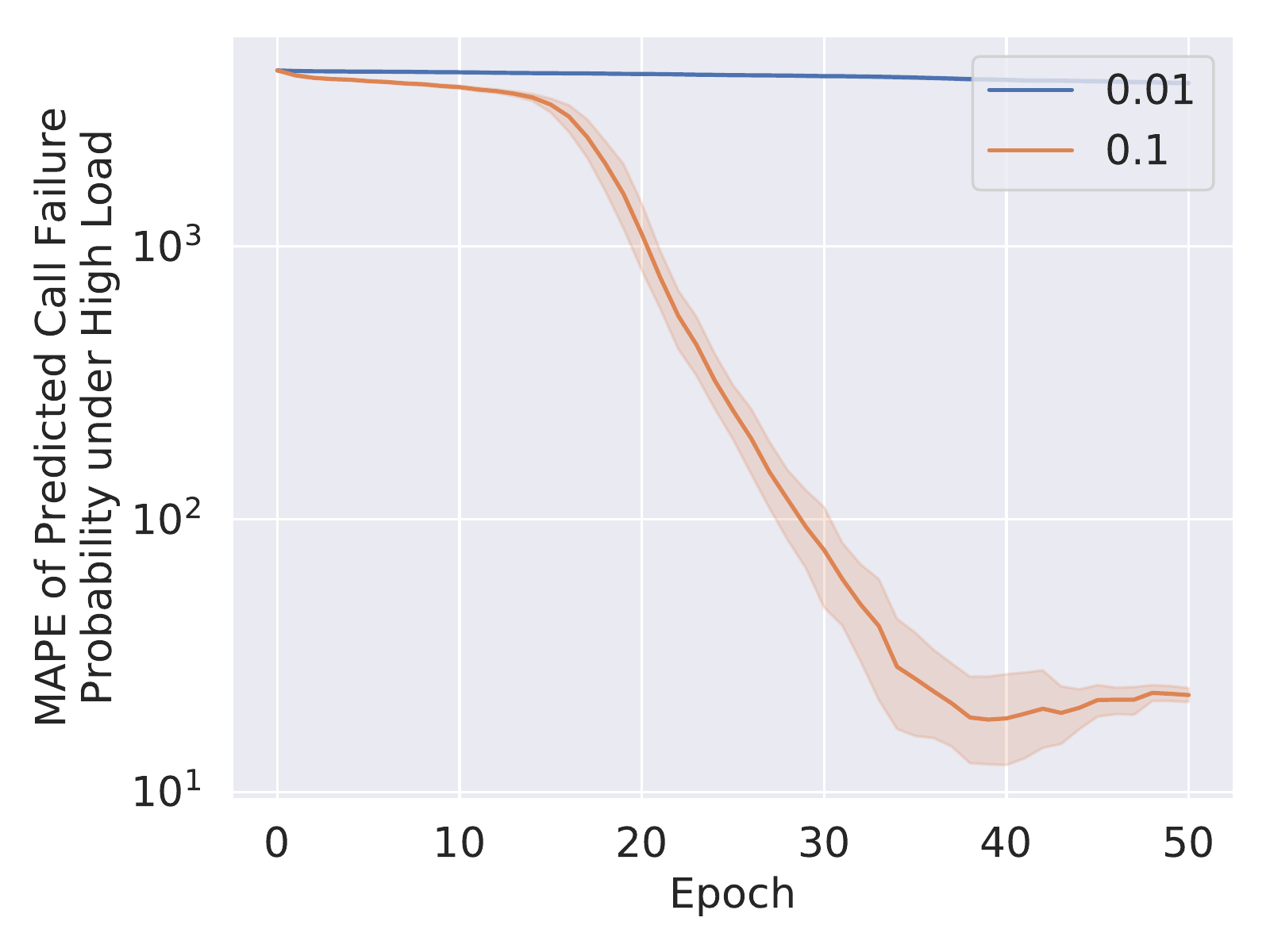}
\subcaption{Effect of geometric sampling on M/M/Mutiple/$K$ ($p \in \{0.01,0.1\}$).}
\label{fig:mmul-large-geop}
\end{minipage}
\caption{Effect of geometric sampling on different queues on test error (all optimized with $\infty$-SGD). (a-d) together show that taking geometric sampling with $p=0.1$ gives nearly the same performance as $p=0.01$.}
\label{fig:geop-study}
\end{figure}

\begin{table*}
\vspace{-5pt}
\caption{\small Simulation results showing training loss (negative log-likelihood) and variance on failure states are of the same magnitude between difference $p$ settings. \label{tab:geop-mse}}
\centering
\resizebox{1.\textwidth}{!}{
\begin{tabular}{@{\extracolsep{0pt}}lcll@{}}
& $\delta_n$ (spectral gap) & $p = 0.1$ & $p=0.01$ \\
\cmidrule(l){2-4}
M/M/1/$K$ (fast-mix) & [0.022, 0.043] & $ 3.87 \times 10^{1} \pm 1.96 ~~~~~~~~~~~ $ & $ 3.59 \times 10^{1} \pm 7.33 \times 10^{-1} $ \\
M/M/1/$K$ (slow-mix) & [0.005, 0.008] & $ 8.25 ~~~~~~~~~~~ \pm 3.75 \times 10^{-1} $ & $  8.11 ~~~~~~~~~~~ \pm 3.42 \times 10^{-1} $ \\
M/M/$m$/$m+r$        & [0.013, 0.024] & $ 2.21 \times 10^{1} \pm 1.68 ~~~~~~~~~~~ $ & $ 2.22 \times 10^{1} \pm 1.64 ~~~~~~~~~~~ $ \\
M/M/Multiple/$K$     & [0.068, 0.096] & $ 1.12 \times 10^{2} \pm 8.88 \times 10^{-1} $ & $ 9.30 \times 10^{1} \pm 1.46 ~~~~~~~~~~~ $ \\
\hline
\end{tabular}
}
\vspace{-10pt} 
\end{table*}

%
%
%
%
\subsection*{\AppExpE: Other Synthetic Result Details} \label{subsec:app:other}

\paragraph{Training Curves.}
Looking at the training curves of synthetic results in Figure~\ref{fig:mape} and Figure~\ref{fig:mse}, we can clearly see that DC-BPTT has a gradient vanishing problem, while  $\infty$-SGD always computes useful gradients. 

\begin{figure*}[h!!!]
\begin{minipage}{0.24\textwidth}
\centering
\includegraphics[width=1.8in,height=1.5in]{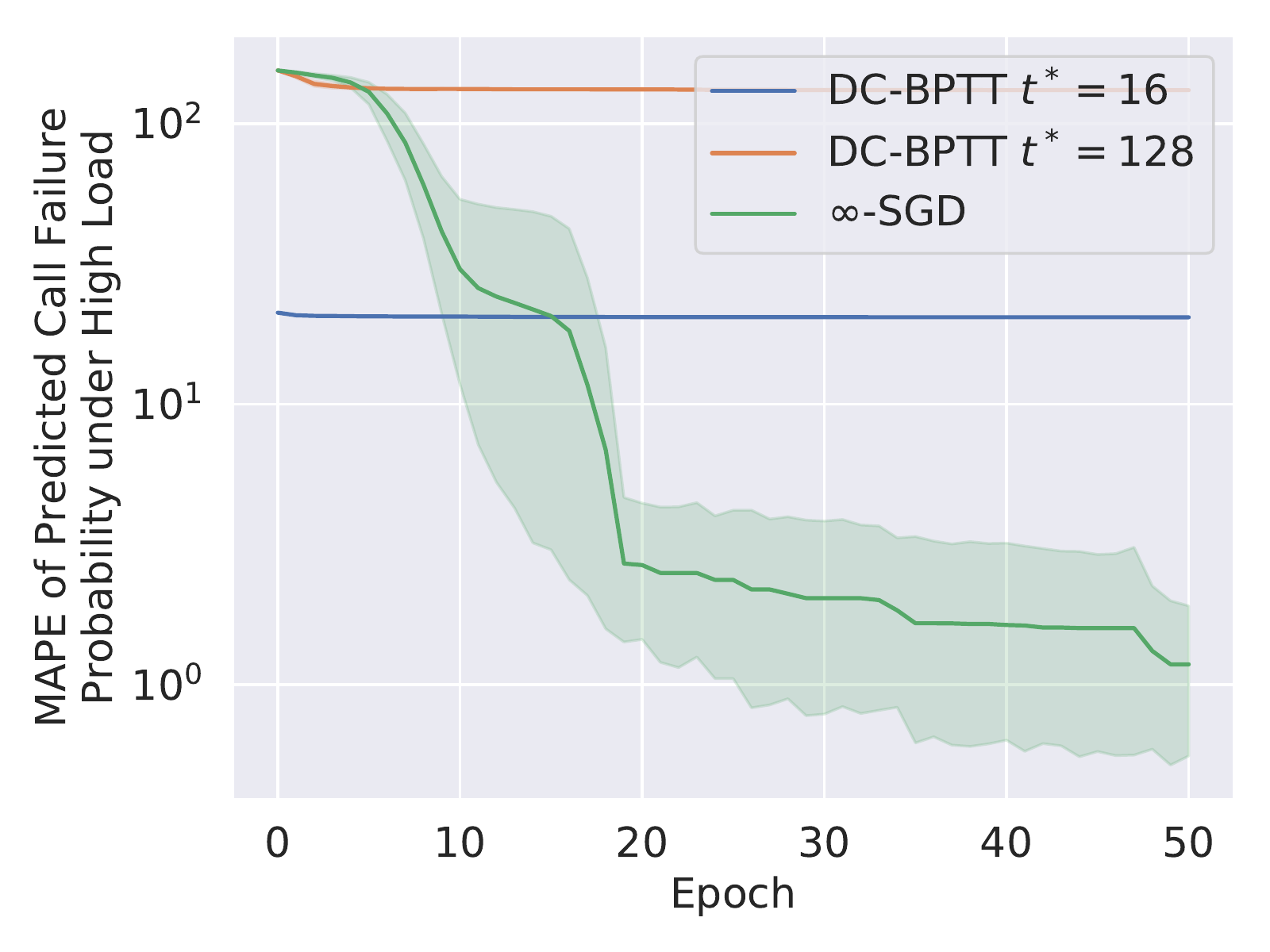}
\subcaption{M/M/1/$K$ (fast-mix)}
\end{minipage}
\hfill
\begin{minipage}{0.24\textwidth}
\centering
\includegraphics[width=1.8in,height=1.5in]{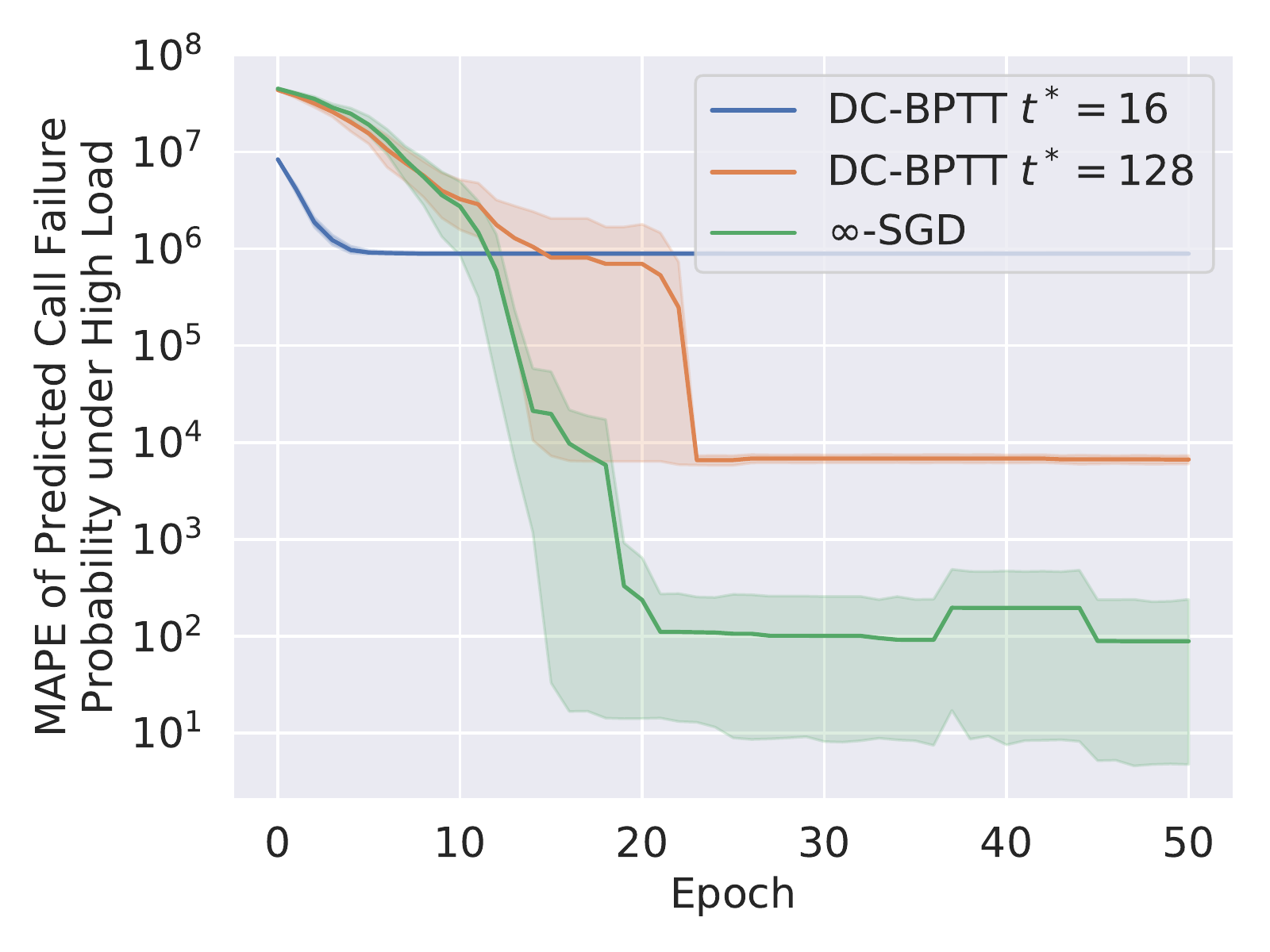}
\subcaption{M/M/1/$K$ (slow-mix)}
\end{minipage}
\hfill
\begin{minipage}{0.24\textwidth}
\centering
\includegraphics[width=1.8in,height=1.5in]{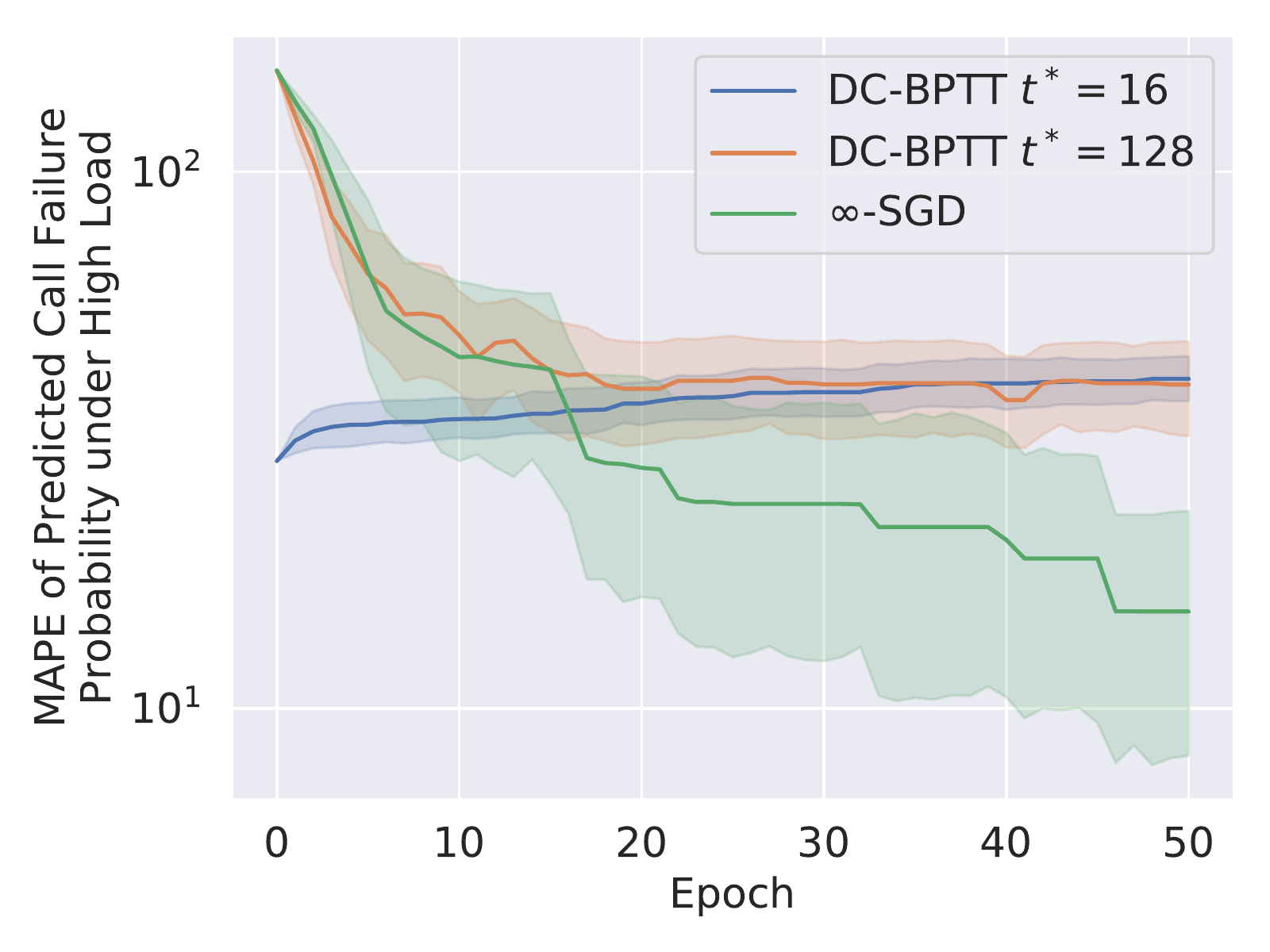}
\subcaption{M/M/$m$/$m+r$}
\end{minipage}
\hfill
\begin{minipage}{0.24\textwidth}
\centering
\includegraphics[width=1.8in,height=1.5in]{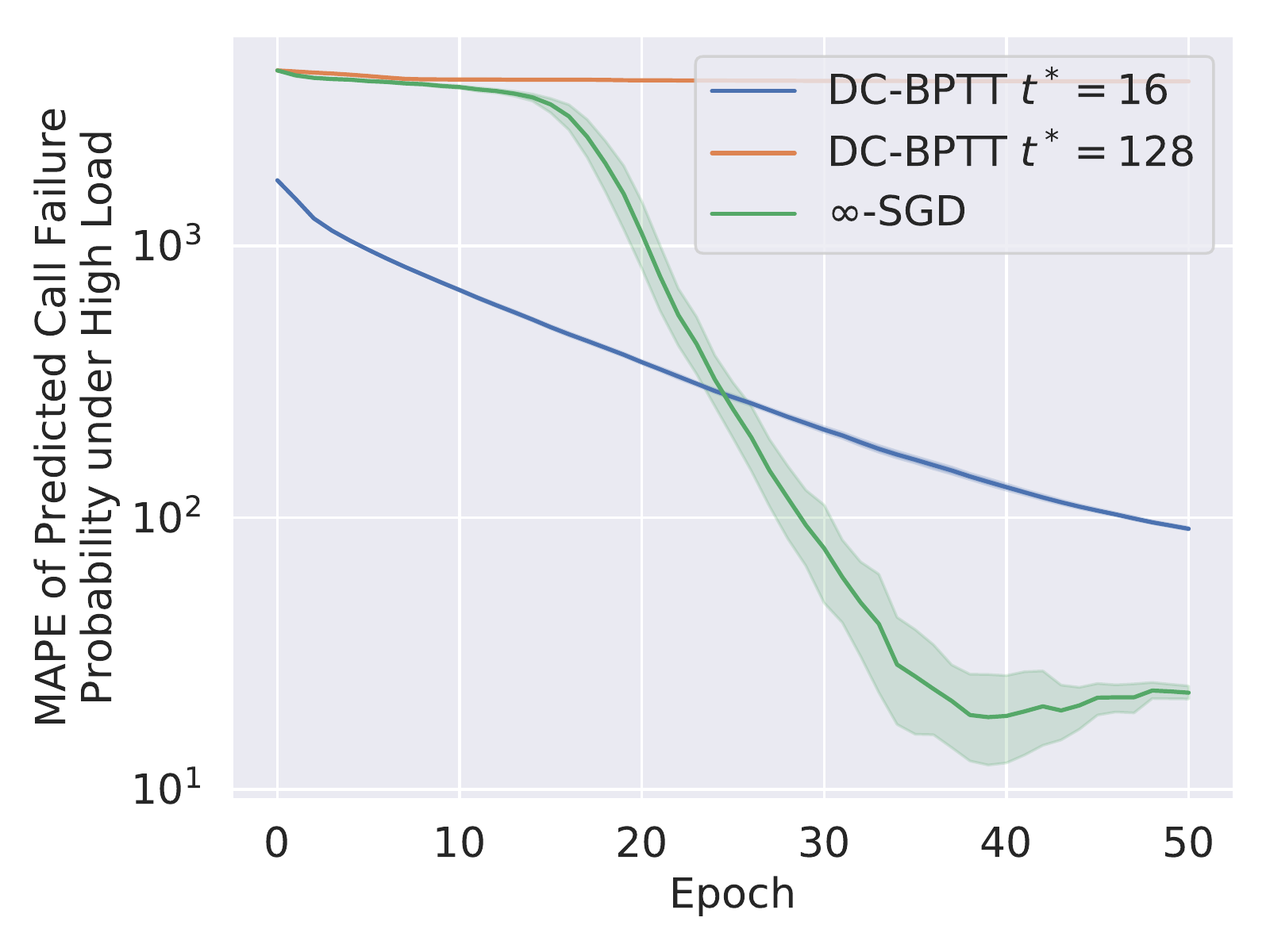}
\subcaption{M/M/Multiple/$K$}
\end{minipage} 
\caption{\small Training curves of MAPE on synthetic experiments.}
\label{fig:mape}
\vspace{-10pt} 
\end{figure*}

\begin{figure*}[h!!!]
\begin{minipage}{0.24\textwidth}
\centering
\includegraphics[width=1.8in,height=1.5in]{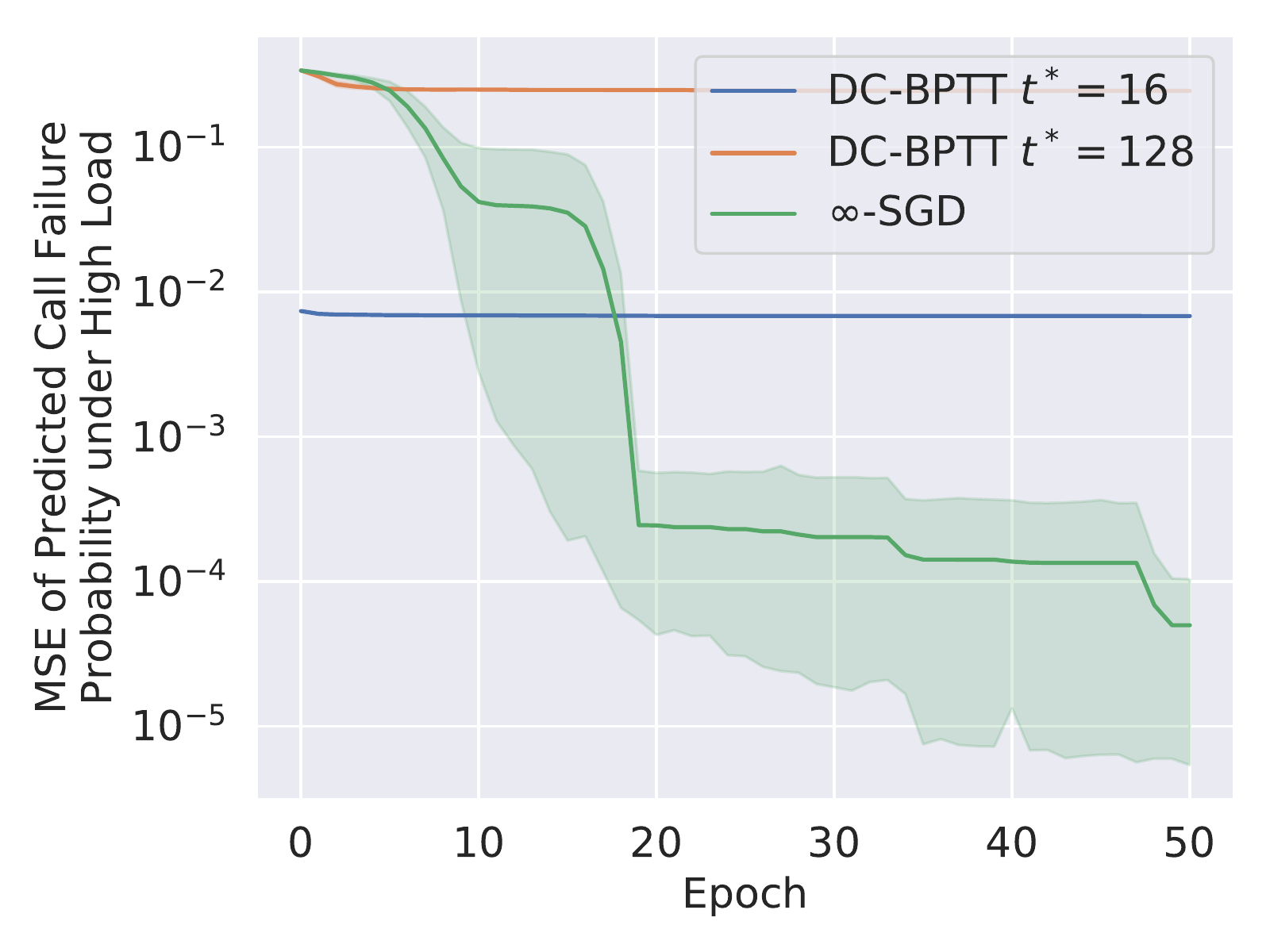}
\subcaption{M/M/1/$K$}
\end{minipage}
\hfill
\begin{minipage}{0.24\textwidth}
\centering
\includegraphics[width=1.8in,height=1.5in]{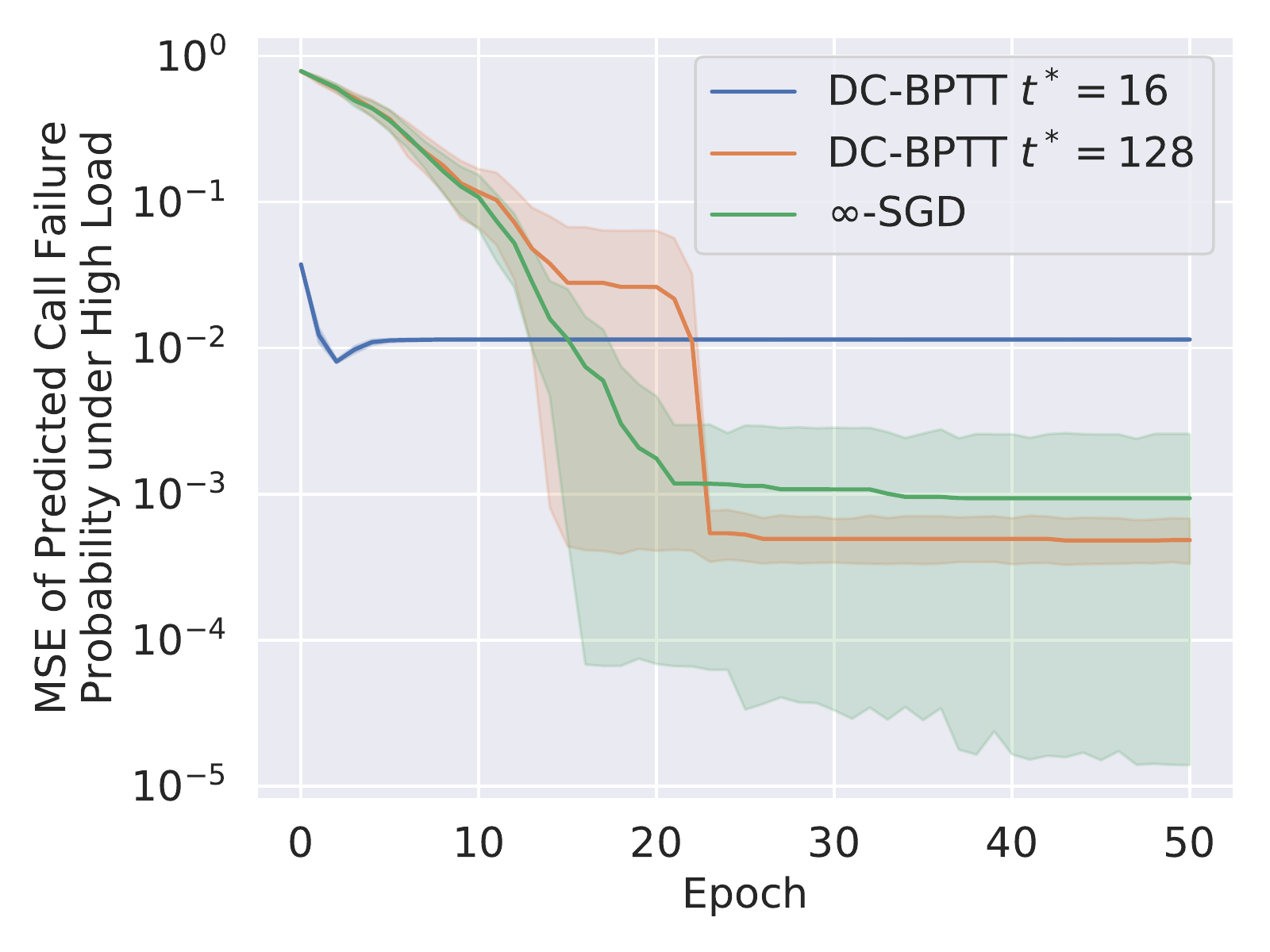}
\subcaption{M/M/1/$K$ (slow-mix)}
\end{minipage}
\hfill
\begin{minipage}{0.24\textwidth}
\centering
\includegraphics[width=1.8in,height=1.5in]{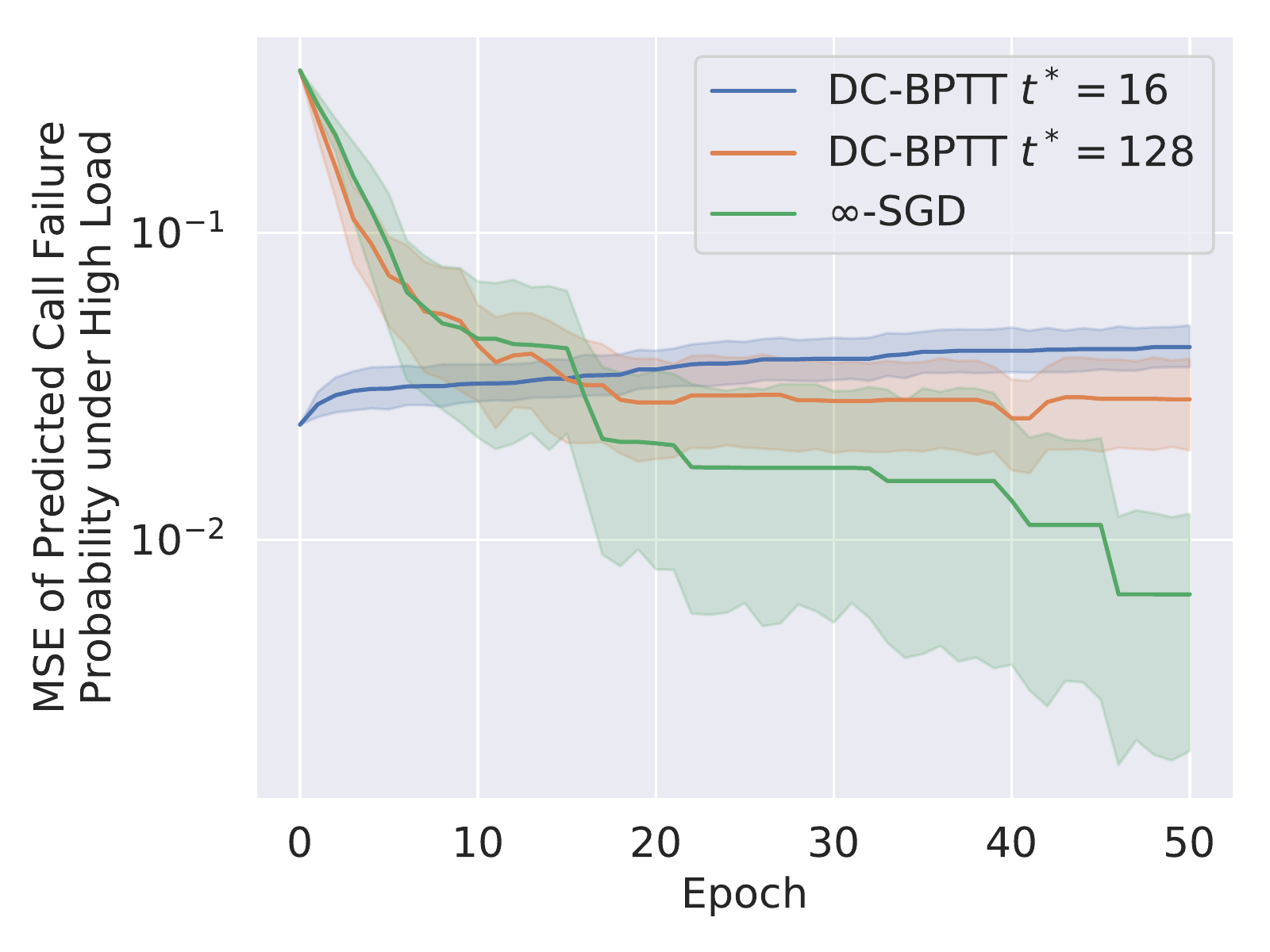}
\subcaption{M/M/$m$/$m+r$}
\end{minipage}
\hfill
\begin{minipage}{0.24\textwidth}
\centering
\includegraphics[width=1.8in,height=1.5in]{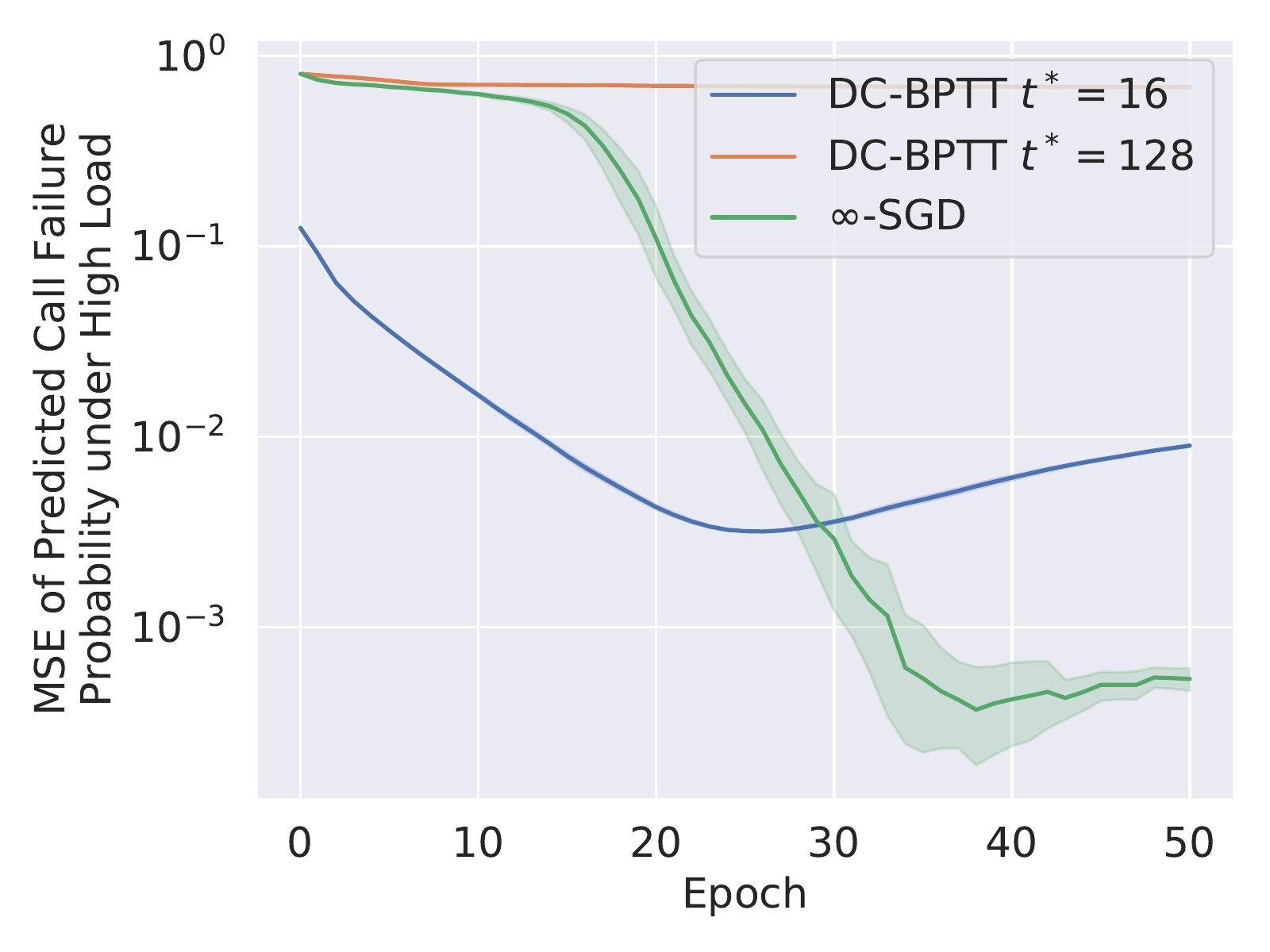}
\subcaption{M/M/Multiple/$K$}
\end{minipage} 
\caption{\small Training curves of MSE on synthetic experiments.}
\label{fig:mse}
\vspace{-10pt} 
\end{figure*}

\paragraph{Other Synthetic Configurations.}
In our experiments, we also tried some variations of the simulated transition rate parameters. Here, we control the spectral gap on training and test data to have a better understanding of our methods on slow mixing events. These transition rates are listed in Table~\ref{tab:paramsimul}.

Final results on those configurations are provided in Table~\ref{tab:mape2} (MAPE) and Table~\ref{tab:mse2} (MSE). We can see that $\infty$-SGD still outperforms DC-BPTT on MAPE metric, while is only comparable to DC-BPTT $t^\star = 16$ on MSE metric in the M/M/m/m+r case. Thus, $\infty$-SGD shows to be a much more stable optimization method. 

\begin{table*}[t!!]
\caption{Parameter configurations for synthetic training and test data whose spectral gap is controlled. \label{tab:paramsimul}}
\centering
\begin{tabular}{lccccc}
& & \multicolumn{2}{c}{Training} & \multicolumn{2}{c}{Test} \\
\cmidrule(l){3-4} \cmidrule(l){5-6} \rule{0pt}{1.0\normalbaselineskip}
& $\mu$ & $\lambda^\text{train}_\text{min}$ & $\lambda^\text{train}_\text{max}$ & $\lambda^\text{test}_\text{min}$ & $\lambda^\text{test}_\text{max}$ \\
M/M/1/$K$ (Gap Controlled)        & 25          & 11 & 40 & 11 & 40 \\
M/M/m/m+r (Gap Controlled)        & 25          & 21 & 30 & 11 & 40 \\
M/M/Multiple/$K$ (Gap Controlled) & [15, 10, 5] & 16 & 45 & 16 & 45 \\
\hline
\end{tabular}
\end{table*}

\begin{table*}[h!!!]
\vspace{-5pt}
\caption{\small [MAPE] Simulation results showing MAPE/100 errors between predicted steady state and ground-truth for failure states in data whose spectral gap is controlled. \label{tab:mape2}}
\vspace{-5pt}
\centering
\resizebox{1.\textwidth}{!}{
\begin{tabular}{@{\extracolsep{0pt}}lclll@{}}
& $\delta_n$ (spectral gap) & \multicolumn{1}{c}{DC-BPTT $t^\star = 16$} & \multicolumn{1}{c}{DC-BPTT $t^\star = 128$} & \multicolumn{1}{c}{$\infty$-SGD ($p=0.1$)} \\
\cmidrule(l){2-5}
 M/M/1/$K$ (gap controlled) & $[0.005, 0.032]$        & $ 2.55 \times 10^{4} \pm 3.37 \times 10^{3} $ & $ 9.06 \times 10^{4} \pm 5.03 \times 10^{4} $ & $ {\bf 6.06 \times 10^{-1} \pm 5.91 \times 10^{-1}} $ \\
 M/M/$m$/$m+r$ (gap controlled) & $[0.010, 0.027]$  & $ 1.32 \times 10^{3} \pm 3.01 \times 10^{1} $ & $ 1.31 \times 10^{2} \pm 4.27 \times 10^{1} $ & $ {\bf 5.14 \times 10^{1} \pm 3.29 \times 10^{1}} $ \\
 M/M/Multiple/$K$ (gap controlled) & $[0.012, 0.057]$ & $ 1.19 \times 10^{3} \pm 1.07 \times 10^{2} $ & $ 2.88 \times 10^{5} \pm 9.11 \times 10^{3} $ & $ {\bf  6.69 \times 10^{-1} \pm 7.00 \times 10^{-1} } $ \\
\hline
\end{tabular}
}
\end{table*}

\begin{table*}[h!!!]
\caption{\small [MSE] Simulation results showing MSE errors between predicted steady state and ground-truth for failure states in data whose spectral gap is controlled. \label{tab:mse2}}
\vspace{-5pt}
\centering
\resizebox{1.\textwidth}{!}{
\begin{tabular}{@{\extracolsep{0pt}}lclll@{}}
\rule{0pt}{1.0\normalbaselineskip} & \multicolumn{1}{c}{$\delta_n$ (spectral gap)}
& \multicolumn{1}{c}{DC-BPTT $t^\star = 16$} & \multicolumn{1}{c}{DC-BPTT $t^\star = 128$} & \multicolumn{1}{c}{$\infty$-SGD ($p=0.1$)} \\
\cmidrule(l){2-5}
 M/M/1/$K$ (gap controlled) & $[0.005, 0.032]$        & $ 8.56 \times 10^{-3} \pm 9.92 \times 10^{-5} $ & $ 2.13 \times 10^{-1} \pm 1.15 \times 10^{-1} $ & $ {\bf 5.00 \times 10^{-4} \pm 6.28 \times 10^{-4}} $ \\
 M/M/$m$/$m+r$ (gap controlled) & $[0.010, 0.027]$ & $ {\bf 2.27 \times 10^{-2} \pm 9.17 \times 10^{-4}} $ & $ {\bf 4.57 \times 10^{-2} \pm 1.37 \times 10^{-2}} $ & $ {\bf 3.97 \times 10^{-2} \pm 1.61 \times 10^{-2}} $ \\
 M/M/Multiple/$K$ (gap controlled) & $[0.012, 0.057]$ & $ 4.34 \times 10^{-5} \pm 2.89 \times 10^{-6} $ & $ 5.61 \times 10^{-1} \pm 1.24 \times 10^{-2} $ & $ {\bf 4.00 \times 10^{-7} \pm 5.15 \times 10^{-7}} $ \\
\hline
\end{tabular}
}
\vspace{-10pt} 
\end{table*}

\end{document}